\PassOptionsToPackage{dvipsnames}{xcolor}
\documentclass[11pt,letterpaper]{article}
\usepackage[margin=1in]{geometry}

\usepackage[utf8]{inputenc}
\usepackage[T1]{fontenc} %

\usepackage{amsmath, amssymb}
\usepackage{amsthm}
\usepackage{mathtools}
\usepackage{nicefrac}
\usepackage{fullpage}
\usepackage{algorithm}
\usepackage{algpseudocode}[1]

\usepackage[margin=1in]{geometry}
\usepackage[utf8]{inputenc}
\usepackage[T1]{fontenc} %
\usepackage{times}
\usepackage{cancel}
\usepackage{mathtools}
\usepackage{mathpazo}

\usepackage{manfnt} %
\usepackage{figchild}
\usepackage{fontawesome5}

\usepackage{accents}

\definecolor{niceRed}{RGB}{190,38,38}
\definecolor{niceYellow}{HTML}{f5b400}
\definecolor{blueGrotto}{HTML}{059DC0}
\definecolor{royalBlue}{HTML}{057DCD}
\definecolor{navyBlue}{HTML}{0B579C}
\definecolor{yaleBlue}{HTML}{00356b}
\definecolor{limeGreen}{HTML}{81B622}
\definecolor{nicePurple}{HTML}{9c27b0}
\definecolor{lightRoyalBlue}{HTML}{def2ff}  
\definecolor{gold}{HTML}{ffa300}

\usepackage{hyperref}
\hypersetup{
  colorlinks = true, 
  urlcolor = {blueGrotto},
  linkcolor = {royalBlue},
  citecolor = {navyBlue}
}
\usepackage[capitalise,noabbrev,nameinlink,sort]{cleveref}
\usepackage[
  backref=true,
  backend=biber,
  natbib=true,
  style=alphabetic,
  sorting=alphabeticlabel,
  sortcites=true,
  minbibnames=3,
  maxbibnames=999,
  mincitenames=4,
  maxcitenames=4,
  minalphanames=4,
  maxalphanames=4,
  doi=false,
  isbn=false,
]{biblatex}

\DeclareSortingTemplate{alphabeticlabel}{
  \sort[final]{%
    \field{labelalpha} %
  }
  \sort{%
    \field{year}   %
  }
  \sort{%
    \field{title}  %
  }
}

\AtEveryBibitem{%
  \clearname{editor}%
}
\AtEveryBibitem{%
  \clearlist{location}%
} 

\AtBeginRefsection{\GenRefcontextData{sorting=ynt}}
\AtEveryCite{\localrefcontext[sorting=ynt]}
\usepackage[most]{tcolorbox}

\tcbset{
    mybox/.style={
        colframe=black,        %
        colback=gray!0,       %
        boxrule=0.5mm,         %
        width=\linewidth,      %
        arc=3mm,               %
        left=5mm,              %
        right=5mm,             %
        top=5mm,               %
        bottom=5mm,            %
        enhanced,              %
        halign=center        %
    }
}

\usepackage{booktabs} 
\usepackage{multicol} 
\usepackage{multirow} 
\usepackage{makecell} 
\usepackage{longtable}

\usepackage{graphicx}
\usepackage{float} 
\usepackage{tikz}
\usepackage{pgfplots}
\pgfplotsset{compat=1.17}
\usepgfplotslibrary{fillbetween}
\usetikzlibrary{arrows.meta}
\usepackage{setspace}

\tikzset{
  myNodeFlex/.style={
    draw,
    rectangle,
    rounded corners,
    text centered,
    minimum height=1.5em,
  }
}

\tikzset{
  myNode/.style={
    draw,
    rectangle,
    rounded corners,
    text centered,
    minimum height=1.5em,
    minimum width=3cm,
    text width=5cm,    
  }
}

\tikzset{
  myNodeNarrow/.style={
    draw,
    rectangle,
    rounded corners,
    text centered,
    minimum height=1.5em,
    minimum width=1cm,
  }
}

\tikzset{
  myNodeWide/.style={
    draw,
    rectangle,
    rounded corners,
    text centered,
    minimum height=1.5em,
    minimum width=6cm,
  }
}

\usepackage{xinttools} %
\usetikzlibrary{calc}

\def\biglen{20cm} %
\tikzset{
  half plane/.style={ to path={
       ($(\tikztostart)!.5!(\tikztotarget)!#1!(\tikztotarget)!\biglen!90:(\tikztotarget)$)
    -- ($(\tikztostart)!.5!(\tikztotarget)!#1!(\tikztotarget)!\biglen!-90:(\tikztotarget)$)
    -- ([turn]0,2*\biglen) -- ([turn]0,2*\biglen) -- cycle}},
  half plane/.default={1pt}
}

\usepackage{physics} %
\usepackage{amsmath}
\usepackage{amssymb}
\usepackage{amsthm}
\usepackage{bbm}
\usepackage{blkarray}
\usepackage{mathtools}
\usepackage{nicefrac}
\usepackage{dsfont}
\usepackage{xspace}
\usepackage{pifont} %

\usepackage{thm-restate} %

\usepackage{enumerate} 
\usepackage[inline,shortlabels]{enumitem}
\usepackage{typed-checklist}
\usepackage{tcolorbox}
\usepackage[framemethod=TikZ]{mdframed}
\mdfsetup{%
backgroundcolor=white, %
roundcorner=4pt,
linewidth=1pt}
\usepackage{changepage} %

\usepackage[linesnumbered,ruled,vlined,noend,algo2e]{algorithm2e}

\usepackage{mathrsfs} %
\usepackage{soul} %

\theoremstyle{plain} 
\newtheorem{theorem}{Theorem}[section]
\newtheorem{corollary}[theorem]{Corollary}
\newtheorem{proposition}[theorem]{Proposition}
\newtheorem{lemma}[theorem]{Lemma}

\newtheorem{definition}{Definition}
\newtheorem*{definition*}{Definition}

\theoremstyle{definition} 

\newtheorem{remark}[theorem]{Remark}

\theoremstyle{remark}

\AfterEndEnvironment{definition}{\noindent\ignorespaces}
\AfterEndEnvironment{example}{\noindent\ignorespaces}
\AfterEndEnvironment{assumption}{\noindent\ignorespaces}
\AfterEndEnvironment{lemma}{\noindent\ignorespaces}
\AfterEndEnvironment{theorem}{\noindent\ignorespaces}
\AfterEndEnvironment{proposition}{\noindent\ignorespaces}
\AfterEndEnvironment{fact}{\noindent\ignorespaces}
\AfterEndEnvironment{question}{\noindent\ignorespaces}
\AfterEndEnvironment{corollary}{\noindent\ignorespaces}
\AfterEndEnvironment{model}{\noindent\ignorespaces}
\AfterEndEnvironment{remark}{\noindent\ignorespaces}
\AfterEndEnvironment{proof}{\noindent\ignorespaces}
\AfterEndEnvironment{fact}{\noindent\ignorespaces}
\AfterEndEnvironment{minftheorem}{\noindent\ignorespaces}
\AfterEndEnvironment{inftheorem}{\noindent\ignorespaces}
\AfterEndEnvironment{maintheorem}{\noindent\ignorespaces}
\AfterEndEnvironment{restatable}{\noindent\ignorespaces}
\AfterEndEnvironment{infassumption}{\noindent\ignorespaces}
\AfterEndEnvironment{conjecture}{\noindent\ignorespaces}
\AfterEndEnvironment{claim}{\noindent\ignorespaces}

\crefname{section}{Section}{Sections}
\crefname{theorem}{Theorem}{Theorems}
\crefname{lemma}{Lemma}{Lemmas}
\crefname{definition}{Definition}{Definitions}
\crefname{conjecture}{Conjecture}{Conjectures}
\crefname{corollary}{Corollary}{Corollaries}
\crefname{construction}{Construction}{Constructions}
\crefname{conjecture}{Conjecture}{Conjectures}
\crefname{claim}{Claim}{Claims}
\crefname{observation}{Observation}{Observations}
\crefname{proposition}{Proposition}{Propositions}
\crefname{fact}{Fact}{Facts}
\crefname{question}{Question}{Questions}
\crefname{problem}{Problem}{Problems}
\crefname{remark}{Remark}{Remarks}
\crefname{model}{Model}{Models}
\crefname{example}{Example}{Examples}
\crefname{equation}{Equation}{Equations}
\crefname{appendix}{Appendix}{Appendices}
\crefname{algorithm}{Algorithm}{Algorithms}
\crefname{algocf}{Algorithm}{Algorithms}
\crefname{model}{Model}{Models}
\crefname{figure}{Figure}{Figures}
\crefname{infdefinition}{Informal Definition}{Informal Definitions}
\crefname{inftheorem}{Informal Theorem}{Informal Theorems}
\crefname{infassumption}{Informal Assumption}{Informal Assumptions}
\crefname{minftheorem}{Main Informal Theorem}{Main Informal Theorems}
\crefname{maintheorem}{Main Theorem}{Main Theorems}
\crefname{assumption}{Assumption}{Assumptions}
\crefname{case}{Case}{Cases}
\crefname{program}{Program}{Programs}
\crefname{inequality}{Inequality}{Inequalities}

\newlist{asmpenum}{enumerate}{1} %
\setlist[asmpenum]{label={\arabic*.},ref=\theassumption.{\arabic*}}
\crefname{asmpenumi}{Assumption}{Assumptions}

\usepackage{etoolbox}

\makeatletter
\renewcommand{\eqref}[1]{\textup{\eqrefform@{\ref{#1}}}}
\let\eqrefform@\tagform@

\newcommand{\changetag}[1]{%
  \renewcommand\tagform@[1]{\maketag@@@{(\ignorespaces#1\unskip\@@italiccorr)}}%
}
\makeatother

\makeatletter
\newcommand{\tagnum}[2]{%
    \refstepcounter{equation}%
    \tag{#1) \ (\theequation}%
    \protected@write \@auxout {}{%
        \string \newlabel {#2}{{\theequation}{\thepage}{}{equation.\theequation}{}}%
    }%
}
\makeatother

\newcommand{\qquadtext}[1]{\qquad\text{#1}\qquad}

\newcommand{\qquadand}{\qquadtext{and}}

\def\abs#1{\left| #1 \right|}

\newcommand{\sinbrace}[1]{\ensuremath{\{#1\}}}

\newcommand{\inbrace}[1]{\ensuremath{\left\{#1\right\}}}

\newcommand{\inparen}[1]{\ensuremath{\left(#1\right)}}
\newcommand{\insquare}[1]{\ensuremath{\left[#1\right]}}

\newcommand{\floor}[1]{\ensuremath{\left\lfloor#1\right\rfloor}}
\newcommand{\sfloor}[1]{\ensuremath{\lfloor#1\rfloor}}

\newcommand{\N}{\mathbb{N}}
\newcommand{\R}{\mathbb{R}}

\newcommand{\evF}{\ensuremath{\mathscr{F}}}

\newcommand{\E}{\operatornamewithlimits{\mathbb{E}}}

\newcommand{\sfrac}[2]{{#1/#2}} 
\newcommand{\nfrac}[2]{\nicefrac{#1}{#2}}

\newcommand{\supp}{\operatorname{supp}}

\newcommand{\iid}{i.i.d.}

\newcommand{\eps}{\varepsilon}
\renewcommand{\epsilon}{\varepsilon}
\makeatletter
\newcommand*{\tran}{{\mathpalette\@tran{}}}
\newcommand*{\@tran}[2]{\raisebox{\depth}{$\m@th#1\intercal$}}
\makeatother

\mathchardef\NABLA"272
\newcommand*{\Nabla}{\boldsymbol\NABLA}
\let\nabla\Nabla

\renewcommand{\bar}{\overline}

\renewcommand{\tilde}{\widetilde}

\newcommand{\customcal}[1]{\euscr{#1}}
\newcommand{\cA}{\customcal{A}}

\newcommand{\cD}{\customcal{D}}
\newcommand{\cE}{\customcal{E}}

\newcommand{\cL}{\customcal{L}}

\newcommand{\cX}{\customcal{X}}
\newcommand{\cY}{\customcal{Y}}
\newcommand{\cZ}{\customcal{Z}}

\DeclareMathAlphabet{\mathdutchcal}{U}{dutchcal}{m}{n}
\SetMathAlphabet{\mathdutchcal}{bold}{U}{dutchcal}{b}{n}
\DeclareMathAlphabet{\mathdutchbcal}{U}{dutchcal}{b}{n}

\DeclareMathAlphabet\urwscr{U}{urwchancal}{b}{n}%
\DeclareMathAlphabet\rsfscr{U}{rsfso}{m}{n}
\DeclareMathAlphabet\euscr{U}{eus}{m}{n}
\DeclareFontEncoding{LS2}{}{}
\DeclareFontSubstitution{LS2}{stix}{m}{n}
\DeclareMathAlphabet\stixcal{LS2}{stixcal}{m} {n}

\renewcommand{\paragraph}[1]{\medskip \noindent\textbf{#1}~}

\newcommand{\eg}{\emph{e.g.}}
\newcommand{\ie}{\emph{i.e.}}

\usepackage{upgreek}
\renewcommand{\gamma}{\upgamma}
\renewcommand{\pi}{\uppi}

\newcommand{\eat}[1]{}
\newcommand{\negLL}{\ensuremath{\mathscr{L}}}

\usepackage{filecontents}
\usepackage{caption}
\usepackage{subcaption}
\usepackage{graphicx}

\renewcommand{\cL}{\negLL}

\usepackage{array}
\newcolumntype{L}[1]{>{\raggedright\let\newline\\\arraybackslash\hspace{0pt}}m{#1}}
\newcolumntype{C}[1]{>{\centering\let\newline\\\arraybackslash\hspace{0pt}}m{#1}}
\newcolumntype{R}[1]{>{\raggedleft\let\newline\\\arraybackslash\hspace{0pt}}m{#1}}

\usepackage{fancyhdr}

\fancypagestyle{custom}{
  \fancyhf{}                %
  \fancyfoot[C]{\textcolor{black}{\fcFishD{0.025}{black}{0.5}}} %

}

\newcommand{\generator}{\mathds{G}}

\newcommand{\cl}{\mathrm{Cl}}

\DeclarePairedDelimiter{\card}{\lvert}{\rvert}

\newcommand{\Lap}{\ensuremath{{\textsf{Lap}}}\xspace}

\newcommand{\DeltaDist}[1]{\Delta\inparen{#1}}

\usepackage[capitalise,noabbrev,nameinlink,sort]{cleveref}

\makeatletter
\@ifundefined{Require}{}{}
\@ifundefined{Ensure}{}{}
\@ifundefined{State}{}{}
\@ifundefined{Return}{}{}
\@ifundefined{For}{\newcommand{\For}{\FOR}}{}
\@ifundefined{EndFor}{}{}
\@ifundefined{If}{\newcommand{\If}{\IF}}{}
\@ifundefined{Else}{}{}
\@ifundefined{EndIf}{}{}
\@ifundefined{While}{}{}
\@ifundefined{EndWhile}{}{}
\@ifundefined{Comment}{}{}
\makeatother

\title{Differentially Private Language Generation\\ and Identification in the Limit}
\author{
        \begin{tabular}{C{7.5cm}C{7.5cm}}
        {\bf Anay Mehrotra}
            & {\bf Grigoris Velegkas}\\
        {Stanford University} 
            & {Google Research}\\
        \mbox{{\href{mailto:anaymehrotra1@gmail.com}{anaymehrotra1@gmail.com}}} 
            &   \mbox{{\href{mailto:gvelegkas@google.com}{gvelegkas@google.com}}}\\[4mm]
        {\bf Xifan Yu}
            & {\bf Felix Zhou}\\
                {Yale University} 
                    & {Yale University}\\
        \mbox{{\href{mailto:xifan.yu@yale.edu}{xifan.yu@yale.edu}}}
            & 
                \mbox{{\href{mailto:felix.zhou@yale.edu}{felix.zhou@yale.edu}}}%
        \end{tabular}
}
\date{}

\begin{document}

\maketitle

\begin{abstract}
    We initiate the study of language generation in the limit, a model recently introduced by \citet{kleinberg2024language}, under the constraint of differential privacy. 
    We consider the \emph{continual release} model, where a generator must eventually output a stream of valid strings while protecting the privacy of the entire input sequence.
    Our first main result is that for countable collections of languages, privacy comes at no qualitative cost: we provide an $\eps$-differentially-private algorithm that generates in the limit from \emph{any} countable collection.
    This stands in contrast to many learning settings where privacy renders learnability impossible.
    However, privacy does impose a quantitative cost: there are finite collections of size $k$ for which uniform private generation requires $\Omega(\nfrac{k}{\eps})$ samples, whereas just one sample suffices non-privately.
    
    We then turn to the harder problem of language \emph{identification} in the limit.
    Here, we show that privacy creates fundamental barriers.
    We prove that no $\eps$-DP algorithm can identify a collection containing two languages with an infinite intersection and a finite set difference, a condition far stronger than the classical non-private characterization of identification.
    Next, we turn to the \emph{stochastic} setting  where the sample strings are sampled \iid{} from a distribution (instead of being generated by an adversary). Here, we show that private identification is possible if and only if the collection is identifiable in the adversarial model.
    Together, our results establish new dimensions along which generation and identification differ and, for identification, a separation between adversarial and stochastic settings induced by privacy constraints.
\end{abstract}

\thispagestyle{empty}

\newpage
\thispagestyle{empty}
\tableofcontents
\newpage

\pagenumbering{arabic}

\section{Introduction} 
    Machine learning systems are increasingly trained on sensitive data.
    Once deployed, a model can be queried, shared, and repurposed in ways that may expose information about individual training records.
    This necessitates systems that are trained with privacy guarantees which remain meaningful both in the presence of public information held by a malicious adversary and downstream post-processing.
    
    Differential privacy (DP)~\citep{dwork2006calibrating} has become the standard formalization of this requirement.
    DP is a stability guarantee for randomized algorithms: informally, it requires that changing a single user record in the training data does not significantly change the distribution over outputs.
    DP has been studied extensively in both practice and theory, and a recurring theme is a privacy--utility trade-off.
    For example, in private PAC learning, pure DP has been investigated in a long line of work (see, \eg{}, \cite{block2024privatePublic,ghazi2021private,bun2020private,alon2019private,kasiviswanathan2011private,fioravanti2024private,hanneke2025privatelist}), revealing several regimes where privacy requires additional samples or even renders learning impossible compared to the non-private setting. 
    For instance, the task of PAC learning simple classes such as one-dimensional thresholds with approximate DP guarantees is already infeasible \citep{alon2019private}.
    
    The recent success of large language models (LLMs) at language generation has brought these questions to the foreground.
    Their training relies on vast text corpora that may contain sensitive data, and interactive querying has been shown to elicit memorized fragments \citep{carlini2021extracting}.
    This has led to growing interest in training and adapting language models with formal privacy guarantees, including DP pretraining and fine-tuning efforts (see, \eg{}, \cite{sinha2025vaultgemma,zhou2025private,yu2024finetuning,li2022strong,mcmahan2018recurrent}).
    These developments motivate a mathematical study of language generation under differential privacy.
    
    We study this question within the recent model of \emph{language generation in the limit} introduced by \citet{kleinberg2024language}.
    This model is motivated by classical adversarial frameworks for learning and identification~\citep{gold1967language,littlestone1988learning}, but it replaces the goal of exact identification with the goal of generation -- producing valid unseen strings from the underlying language.
    The process begins with an adversary selecting a target language $K$ from a known collection $\cL=\{L_1,L_2,\dots\}$ and fixing an enumeration of $K$.\footnote{Formally, an enumeration of $K$ is an infinite sequence $x_1,x_2,\ldots$ (potentially with duplicates) such that $x_i\in K$ for all $i$ and every $x\in K$ appears at some index.}
    At each step $n\geq 1$, the adversary reveals the $n$-th element $x_n$ of the enumeration.
    Having observed the set of examples $S_n=\{x_1,\ldots,x_n\}$, the generator $\generator$ must output a new {string $w_n\notin S_n$ intended to be a valid, unseen element of $K$.}

    \enlargethispage{\baselineskip}
    
    A generator $\generator$ is said to be successful if it learns to \emph{generate from $\cL$ in the limit}:
    for any $K\in \cL$ and any enumeration of $K$, there exists a finite round $n^\star$ such that for all $n\geq n^\star$, the output is always correct and novel, $w_n \in K\setminus S_n.$
    This framework is rooted in Gold's notion of \emph{identification in the limit}~\citep{gold1967language}, which requires the learner to identify the target language exactly.
    While identification is impossible for most nontrivial language collections, \cite{kleinberg2024language} showed that the weaker objective of generation is feasible in striking generality, including for any countable collection of languages.
    This separation has catalyzed a wave of recent work refining the model and its guarantees (\eg{}, \cite{li2024generation,kalavasis2025limits,charikar2024facets,raman2025noisy}); see \cref{sec:relatedworks}.
    Given this context, we investigate the possibility of language generation under differential privacy.

    To study privacy in this setting, it is not enough to protect a \emph{single} output of the generator.
    Language generation is an ongoing interaction: after observing $x_{1:n}$ the generator outputs $w_n$, and the privacy guarantee should apply to the entire transcript of outputs.
    Accordingly, we adopt the \emph{continual release} model of DP~\citep{dwork2010continual,chan2011continual}, which (informally) requires that for any two input streams that differ at exactly one timestep, the joint distribution of the \emph{entire} output stream changes by at most a multiplicative factor of $e^\eps$ (for desired privacy value $\eps>0$).
    This temporal requirement is strictly stronger than one-shot privacy, and even for simple tasks, it is known to induce error that grows with the length of the stream~\citep{jain2023price,cohen2024lower,epasto2025sublinear}.
    In our setting, this challenge is compounded by the fact that the number of rounds until convergence is not known in advance and the stream length is infinite.
    This brings us to the main question studied in this work:
    \vspace{2mm}
    \begin{mdframed}
        \textbf{Q:}~ \emph{Which collections $\cL$ are generatable in the limit under $\eps$-DP in the continual release model?}
    \end{mdframed}
    \noindent As any non-trivial DP algorithm is necessarily randomized, \mbox{we allow failures on probability 0 events.}

    \subsection{Our Contributions}
        Our first result shows that $\eps$-DP language generation is possible for \emph{all} countable collections.
        \begin{restatable}[Private Generation]{theorem}{thmPrivateOnlineGeneration}\label{thm:privateOnlineGeneration}\label{infthm:main}
            For any $\eps > 0$, there is an algorithm $\generator$ (\Cref{alg:generation:approxIntersection}) that, for any countable collection $\cL$, 
            $\generator$ is $\eps$-DP in the continual release model and generates in the limit from $\cL$.
        \end{restatable}

        \noindent Thus, requiring differential privacy even in the stronger continual release model does not make the problem of generation harder, and it remains possible for all countable collections.
        This stands in contrast to many other learning tasks, where imposing differential privacy often introduces a fundamental privacy--utility trade-off. At this level of generality (only requiring generation in the limit), privacy appears to come for ``free'' for language generation. We revisit this observation when we consider sample complexity below.
        \noindent While the above algorithm is able to generate in the limit, the time step $n^\star$ after which it begins generating correctly depends, in general, on the choice of the target language $K$.
        For finite collections, we can avoid this: the next result provides a \emph{uniform} bound on the number of samples required for generation in the limit, independent of the choice of the target language and its enumeration. 

        In the non-private setting, \cite{kleinberg2024language} showed that if $\cL$ has finite size, then $n^\star$ (the time at which the generator starts generating correctly) can be upper bounded by a quantity $n(\cL)$ that \textit{only} depends on the collection $\cL$ and not on the target language $K$ or the adversary's enumeration. Furthermore, \citet{li2024generation} characterized the time $n^\star$ exactly using the notion of closure dimension, defined later on in \cref{def:closureDimension}, which is analogous to how the Littlestone dimension characterizes the mistake bound in online learning. For a language collection $\cL$ of closure dimension $d$, \citep{li2024generation} showed that seeing $n^\star = d+1$ distinct input elements is both necessary and sufficient for uniform generation from $\cL$. Our \Cref{infthm:ubFinite} provides an analogous guarantee in the private setting, which says that if we desire a probability $1-\beta$ of ``success'' by time $n^\star$, then the analogous quantity for us is $n^\star = d + \tilde{O}((\nfrac{k}{\eps})\cdot\log(\nfrac1\beta))$.

        \begin{theorem}[Sample-Complexity Upper Bound; Informal; see \cref{thm:ubFinite}]
            \label{infthm:ubFinite}
                There is an $\eps$-DP continual release algorithm $\generator$ that generates from any finite collection $\cL$ of size $k$ and closure dimension $d$.
                For any $\beta > 0$, the step $n^\star$ after which $\generator$ generates satisfies $n^\star \leq d + \tilde{O}\left(\left(\nfrac{k}{\eps}\right)\log \left(\nfrac{1}{\beta}\right)\right)$ \mbox{with probability $1 - \beta$.}
        \end{theorem}
        Note that the bound on $n^\star$ is independent of the target language and its enumeration.
        The sample complexity's dependence on $d$ is expected as it also arises without requiring privacy.
        Further, the dependence on $\nfrac{k}{\eps}$ in the sample complexity of \cref{infthm:ubFinite} is almost tight: $\nfrac{k}{\eps}$ samples are required to achieve even a success probability of $\nfrac23$,
        as shown in our next result.
        \begin{theorem}[Sample-Complexity Lower Bound; Informal; see \Cref{thm:lbFinite}]\label{thm:informal-lbFinite}
            \label{infthm:lb}
            For any $k,d\in \N$, there is a finite collection $\cL$ of size $k$ with closure dimension $d$ such that if the time step $n^\star$ after which an $\eps$-DP generation algorithm in the continual release model uniformly generates from $\cL$ satisfies $n^\star \leq m$ with probability at least $\nfrac{2}{3}$ independent of the target language and its enumeration, then $m = d+\,\Omega\left(\nfrac{k}{\eps}\right)$.
            Moreover, in the absence of privacy constraints, there is an algorithm that generates after observing $d+1$ elements from the adversary.
        \end{theorem}
        This shows that the dependence on $d+\,\nfrac{k}{\eps}$ is unavoidable for uniform private generation (in the sense of \cref{infthm:ubFinite}).
        In fact, we prove a stronger lower bound that already applies under one-shot $\eps$-DP at a \emph{single} time step (without assuming the stronger continual release requirement).
        Thus, for uniform generation from finite collections, there \emph{is} a privacy--utility trade-off: without privacy, generation can succeed after just $d+1$ samples, whereas with privacy, $d+\,\Theta(\nfrac{k}{\eps})$ samples are necessary. 
        This gap can be made arbitrarily large by increasing the size of the collection $k$ (while keeping $d$ fixed).

        \begin{remark}[Non-Uniform Generation]
            The algorithm in \cref{infthm:main} achieves a stronger guarantee of non-uniform generation \citep{li2024generation} (see \Cref{rem:non-uniform-generation}).
            
        \end{remark}

        \paragraph{Private identification.}
         Since requiring differential privacy for generation does not restrict which collections are generatable,
         it is natural to ask whether the same is true for language \emph{identification} in the limit, as defined by \citet{gold1967language}.
        In this model, an adversary similarly selects a target language $K= L_{i^\star}$ from a known collection $\cL=\{L_1,L_2,\dots\}$ and fixes an enumeration of $K $.
        The only difference is that after the adversary reveals the $n$-th element, the algorithm is required to output an index $i_n$. The algorithm \emph{identifies from $\cL$ in the limit} if there is a finite round $n^\star$ such that for all $n\geq n^\star$, $i_n=i^\star$.

            Our next result shows that under $\eps$-DP, unlike generation, identification becomes much harder to achieve.
            As before, we allow the identification algorithm to fail on an event of probability 0.
            \begin{restatable}[Private Identification Barrier]{theorem}{thmPrivateIdentificationBarriers}
                \label{infthm:identification-barriers}
                If $\cL$ contains two distinct $L_i, L_j$ such that $\abs{L_i \cap L_j} = \infty, \abs{L_i \setminus L_j} < \infty,$ then no $\eps$-DP continual release algorithm (for any $\eps > 0$) can identify $\cL$ in the limit.
            \end{restatable}
            In particular, if $\cL$ contains two languages with $L_i \subseteq L_j$, private identification is impossible.
            Due to this, the above condition turns out to be much stronger than \emph{Angluin's condition} (\cref{def:telltales}), which characterizes non-private identification. 
            Hence, combined with \cref{thm:privateOnlineGeneration}, this yields another separation between identification and generation. %
            We complement this negative result with an algorithm for collections satisfying conditions close to the negation of the above (see \cref{thm:finite_intersection_upper}).

            Finally, we study identification in the stochastic model of \cite{angluin1988identifying}, where the input stream is drawn i.i.d.\ from a distribution supported on the target language. Without privacy, identifiability in the stochastic and adversarial settings coincide and are characterized by Angluin's condition (\cref{def:telltales}). We show this equivalence persists under privacy.
            \begin{restatable}[Private Identification in Stochastic Setting]{theorem}{stocIden}\label{thm:stochastic-identification}
                A countable collection of languages $\cL$ is privately identifiable in the limit under stochastic inputs if and only if it satisfies Angluin's condition.
            \end{restatable} 
            {Together with \Cref{infthm:identification-barriers}, this reveals a separation between adversarial and stochastic identification induced by privacy; a phenomenon absent in the non-private setting \citep{angluin1988identifying,kalavasis2025limits,charikar2025characterization} that may merit further exploration.}

            \begin{remark}[Statistical Rates of Private Generation and Identification]
            Our results and techniques have natural implications for the \emph{statistical} setting
            studied by \citet{kalavasis2025limits} (who, in turn, use the \emph{universal rates} model by \citet*{bousquet2021theory}). 
            In this setting, the algorithm receives an \iid{} sample of size $n$ from a distribution supported $\cD$ on some language $K \in \cL$ and its goal is to generate samples from $K$ or, in the case of identification, identify $K$.
            For generation (respectively identification), the quantity of interest is the probability that the algorithm does not generate from $K$ (respectively identify $K$) as a function of $n$. 
            If this failure probability decays as $C \cdot R(c\cdot n)$, we say that $\cL$ is generatable (respectively identifiable) at rate $R.$  
            Notably, the constants can depend on the distribution and on $\eps$ but not the target language $K\in \cL$. 
            
            Informally, we can show that every countable collection (respectively every collection that satisfies Angluin's condition \citep{angluin1980inductive}) is generatable (respectively identifiable) in the limit at an (almost) exponential rate, where the constants depend on the privacy parameter $\eps$.
            Such transformations from algorithms that succeed in the online setting to algorithms that achieve (almost) exponential rates have also appeared in prior works (\eg{}, \citep{kalavasis2025limits,kalavasis2024characterizations,charikar2025characterization}) and our extensions utilize similar techniques.
            \end{remark}

    \vspace{-7mm}
    
\subsection{Related Works}\label{sec:relatedworks}
        Our contributions draw on two main lines of work: (1) language generation in the limit, and (2) differential privacy under continual release.
        We summarize the most relevant related works below.

        \paragraph{Language generation in the limit.}
            A growing line of work studies a range of questions in the language generation in the limit model and its variants (\eg{}, \cite{li2024generation,kalavasis2025limits,charikar2024facets,raman2025noisy,peale2024,kleinberg2025density,kleinberg2025partial,hanneke2025union,mehrotra2025contamination,charikar2025characterization,karbasi2025impossibility,charikar2026pareto,arenas2025language,anastasopoulos2026safe}).
            Perhaps the most closely related work to ours is that of
            \citet{charikar2024facets,mehrotra2025contamination}, whose algorithms we build upon. Moreover, the notion of uniform generation we explore in our work was proposed by \citet{li2024generation}. We provide a more detailed overview of other
            works in this area in \cref{apx:related-work}.

        \paragraph{Differential privacy under continual release.}
        The continual release model of differential privacy requires algorithms to abide by a strong privacy notion: an observer obtaining \emph{all} outputs of the algorithm must, in essence, learn almost nothing about the existence of any single input. 
        Since its introduction, this research area has received vast attention, including many
        recent works (see \eg{}\ \cite{perrier2019real,fichtenberger2023constant,jain2023counting}).
        This includes classical estimation problems~\cite{chan2011continual,cardoso2022histograms,henzinger2023differentially,henzinger2024unifying}, heavy hitters-related problems~\cite{ChanLSX12,epasto2023differentially}, and lower bounds~\citep{jain2023price,cohen2024lower,epasto2025sublinear}. 

    \vspace{-2mm}
    \section{Technical Overview}
    \vspace{-2mm}

    In this section, we overview the main ideas and challenges in proving our results.
    To explain the challenges that the privacy requirement
    introduces in this setting, we start with \emph{identification}, and then illustrate that we can design generators that do not suffer from these hurdles.

    \vspace{-2mm}

    \subsection{Online Model of Private Identification (\cref{infthm:identification-barriers} and \cref{thm:finite_intersection_upper}) }
        \vspace{-2mm}
        
        \paragraph{Identification lower bound.} 
        We begin with our lower bound, which is more involved than the algorithm. Suppose $\cL$ contains $L_i, L_j$ with $\abs{L_i \cap L_j} = \infty$ and $\abs{L_i \setminus L_j} < \infty$, and assume for contradiction that some algorithm identifies $\inbrace{L_i, L_j}$. Starting from an enumeration $E$ of $L_i$, the algorithm outputs $L_j$ only finitely often with probability one. Using the \emph{group}-privacy
        guarantees and the correctness properties of the algorithm, we
        show how to find a sequence of timesteps $\inbrace{t_{k_\ell}}_{\ell \in \N}$ such that if we swap elements of $E$ appropriately on these timesteps,
        we can \textbf{(i)} convert $E$ to an enumeration $E'$ of $L_j$, 
        and \textbf{(ii)} guarantee that the algorithm \emph{cannot} identify
        $L_j$ in this enumeration. The technical details to make this
        work are involved since we need to make infinitely many swaps from $E$ to turn it to an enumeration of $L_j$, while ensuring the algorithm makes
        infinitely many mistakes. The proof appears in \cref{apx:proofs:privateOnlineIdentificationLowerBound}.

        \paragraph{Identification algorithm.} 
        {Next, we describe an algorithm that identifies in the limit any countable collection in which every pair of distinct languages has finite intersection; intuitively, the languages are \emph{almost} disjoint and share only finitely many elements. For intuition, consider two languages $\inbrace{L_1, L_2}$ with this property. For each $L_i$, maintain an error counter that is equal to the number of stream elements it misses. Then, for any adversarial stream,\footnote{This holds even if we allow each element to be repeated a constant amount of times.} exactly one counter stays at zero while the other grows \emph{linearly} in the limit. Now, standard continual-release techniques~\citep{dwork2010continual} let us distinguish the two languages. We extend this idea to countable collections by restricting the active search space to \emph{finitely} many candidate languages at each timestep, which lets us bound the error probability via union bounds.}
 
        \vspace{-2mm}
        
        \subsection{Stochastic Model of Private Identification (\cref{thm:stochastic-identification})}  
        We now turn to the stochastic setting of private identification. 
        
        To design a private algorithm here, a natural approach is to ``privatize'' an off-the-shelf identification algorithm, like the one from \citet{angluin1980inductive}. 
        Unfortunately, it is not clear how to do that since these algorithms heavily rely on keeping track of a \emph{version space}, \ie{}, the set of all consistent languages with the current stream of examples, which can change dramatically on swapping just one element in the stream.
        
        To circumvent these, we use the \emph{exponential mechanism} \cite{talwar2007exponential}; the  main technical hurdles are to \textbf{(i)} design appropriate score functions with low sensitivity, and  \textbf{(ii)} since the output space is infinite, the tail of the distribution induced by the exponential mechanism needs to decay sufficiently fast.
        Intuitively, our scoring function has two components; the first penalizes languages that are \emph{not} supersets of $K$ and the second penalizes languages that are (strict) supersets of $K$.
        The former can be easily achieved by counting how many stream elements each language misses. To achieve the latter, we show it suffices to penalize a language when its \emph{tell-tale} (\cref{def:telltales}) has not yet appeared in the stream. 
        We design such a function with small sensitivity which, crucially, has the property that in the stochastic setting we can \emph{lower bound} the rate at which it is decreasing for all $L_i \neq K.$ 
        This separation is what allows privacy in the stochastic model without additional requirements, while the online setting has a high cost of privacy.
                
        To ensure that the tail of the (exponential) distribution decays sufficiently fast and we do not exceed our privacy budget, we run the algorithm in \emph{epochs} of exponentially increasing size and perform ``lazy updates,'' \ie{}, the output remains the same for all timesteps in a given epoch. 
        We sample each language $L_i$ with probability proportional to $\pi_t(i)\cdot \exp(\lambda u_t(i)),$ where $u_t$ is the scoring function, $\pi_t$ is a \emph{data-independent} base measure that heavily downweights languages with large indices, and changes across epochs, and $\lambda$ is related to the sensitivity of $u_t$ and the privacy budget.
        By carefully choosing all the underlying parameters we can show that the sum of the error probabilities across epochs is finite, thus implying only finitely many identification mistakes almost surely through the Borel--Cantelli lemma (\cref{lem:first-Borel--Cantelli}). 

        \vspace{-2mm}
        
        \subsection{Private Generation (\cref{infthm:main})}\label{sec:overview:privateGeneration}
        Having illustrated the inherent limitations of private identification, we now explain why generation avoids these obstacles. Recall that if $L_i \subsetneq L_j$, private (online) identification is impossible even for the two-language class $\{L_i,L_j\}$. In contrast, private generation is trivial in this case: since $L_i \cap L_j$ is infinite, a generator can safely output elements from this intersection for infinitely many timesteps.
        This idea also underlies the generators of \citet{kleinberg2024language} (and \citet{charikar2024facets}). Thus, a natural route is to try to use the exponential mechanism~\citep{talwar2007exponential} to privatize these algorithms.
        Unfortunately, similar to the identification case, these algorithms are very brittle since they require tracking the version space.

        \paragraph{Our approach.} We instead build on the recent algorithm of \citet{mehrotra2025contamination} (inspired by \citet{charikar2024facets}), which is more amenable to privatization because it does not explicitly maintain a version space. Instead, the algorithm assigns each language a priority based on the number of inconsistent strings seen so far, and then (following this priority order) forms incremental intersections until the intersection remains infinite. A careful analysis of the high-priority languages shows that the target language $K$ must eventually be a part of the maintained intersection.
        Crucially, the algorithm accesses the stream \textit{only} through these priorities. 
        We can privatize the priority computation at a \emph{single} timestep via the Laplace mechanism, and then repeat this at sparse timesteps while allocating the privacy budget across repetitions to obtain continual-release guarantees. This is reminiscent of the lazy-updates paradigm from continual-release graph algorithms~\citep{fichtenberger2021continual,epasto2025sublinear,dinitz2025differentially,zhou2026continual}. It remains to show that the resulting noisy priorities are accurate enough that $K$ is included in the intersection with probability $1$.
        Once we have computed this infinite subset $U\subseteq K$, generating an unseen element can be accomplished by truncating this set at a sufficiently \mbox{long prefix and sampling an element uniformly.}

         \vspace{-2mm}

        \subsection{Sample Complexity of Private Generation (\cref{infthm:ubFinite,infthm:lb})}\label{sec:overview:privateUniformGeneration}
        We now study the sample complexity of private generation under \emph{uniform} bounds, meaning bounds that do not depend on the target language $K$ and its enumeration.
        The analysis in this setting turns out to be significantly more delicate
        than the previous one. Without privacy, such uniform bounds exist if and only if $\cL$ has finite closure dimension (\cref{def:closureDimension}).

        \paragraph{Sample complexity upper bound (\cref{infthm:ubFinite}).} 
            We begin with \emph{finite} collections, which admit uniform bounds in the non-private setting \citep{kleinberg2024language}. Since the algorithm from the previous subsection does not exploit finiteness, we analyze a different procedure here.\footnote{Note that while our algorithm here will be able to achieve a uniform sample complexity, it is incomparable to the algorithm in the previous subsection result since the current algorithm does not generate from all countable collections.}

            A simple (non-private) algorithm for uniformly generating finite collections is as follows: output the smallest unseen element from the closure (\ie{}, intersection) of all consistent languages, where a language $L$  is consistent if $L\supseteq S_n$.
            To prove \cref{infthm:ubFinite}, we show that this algorithm can be privatized via the exponential mechanism with a carefully designed score.
            To be more precise, our score function will assign scores to subsets of languages, and our algorithm will sample a subset $S$ of languages and output their closure $\cl(\cL_S\coloneqq\inbrace{L_i\colon i\in S})$.\footnote{Given this closure, one can always privately post-process to sample one unseen element from it; \cref{lem:elementGenerationFromSetGeneration}.}
            We will design a score function which comes with the guarantee that, as $n\to \infty$, with probability 1, the sampled subcollection $\cL_S$ (P1) contains $K$ and (P2) $\cl(\cL_S)$ is infinite.

            Achieving Property (P2) is straightforward: it suffices to ensure that $\cl(\cL_S)$ contains at least $d+1$ elements, where $d$ is the closure dimension of $\cL$. Then the definition of closure dimension implies $\card{\cl(\cL_S)}=\infty$ \citep{li2024generation}.
            The main work is establishing (P1).
            A simple score rewards $\cL_S$ proportional to how many enumerated elements lie in $\cl(\cL_S)$, but this does not differentiate between $K$ and its supersets. So any superset of $K$ has the same score and, hence, the same probability of being sampled as $K$. So the probability of sampling $K$ can be as small as $1/c$, where $c$ is the number of supersets of $K$ in $\cL$. 
            One could repeat the exponential mechanism $t_n$ times to amplify probability of sampling $K$, but this would require $t_n \to \infty$ with $n$ and would incur additional privacy loss with each re-sampling.

            Instead, we design a different score function which balances two competing goals: (G1) favoring \textit{larger} subcollections and (G2) favoring subcollections whose closure contains more elements from the input enumeration.
            The key observation is simple: if $K \notin \cL_S$, then adding $K$ yields a subcollection that weakly improves both (G1) (it is larger) and (G2) (including $K$ does not remove any elements from closure).
            We show that observation is enough to conclude that, with sufficiently high probability in $n$, the exponential mechanism will sample a subcollection that contains $K$.

        \paragraph{Sample complexity lower bound (\cref{infthm:lb}).} 
        Having proved an upper bound for finite collections, it is natural to ask whether it is tight and whether a similar guarantee extends to all countable collections with finite closure dimension. We show the upper bound is tight, and moreover that there exist collections with closure dimension zero that still do not admit any uniform private bound.
        Our lower bound uses the standard packing lower bound approach for DP~\citep{hardt2010geometry}.
        This framework proceeds roughly as follows.
        Let $M: \cX^n\to [N]$ be an $\eps$-DP mechanism with discrete output space $[N]$ and suppose that every $v\in [N]$ is the unique correct answer to $M(X')$ for some $X'\in \cX^n$.
        For any dataset $X\in \cX^n$, there must be at least one output $v\in [N]$ such that $\Pr[M(X) = v]\leq \nfrac1N$.
        By assumption, there is some $X'\in \cX^n$ where $\Pr[M(X') = v]\geq \nicefrac23$ since $v$ is the uniquely correct response for dataset $X'$.
        By the definition of DP,
        $
            \nfrac23
            \leq \Pr[M(X') = v]
            \leq e^{n\eps}\cdot \Pr[M(X) = v]
            \leq \nfrac{e^{n\eps}}{N}\,.
        $
        In other words, $n\geq \Omega(\nfrac{(\log N)}\eps)$.

        In our lower bound construction, by an appropriate postprocessing we may take the relevant output space to be a subset of the $2^k$ index sets $I\subseteq[k]$, each encoding an infinite intersection $\bigcap_{i\in I} L_i$ of languages from a size-$k$ collection.
        The main technical challenge is to construct a size $k$ collection that ``packs'' as many different unique correct responses as possible for input streams of length $n$.
        We do so via a Sperner family, which provides $N=\tilde\Omega(2^k)$ distinct responses and thus gives the desired lower bound.

        \vspace{-2mm}

\section{Model and Preliminaries}
In this section, we introduce differential privacy and the model of language generation in the limit.

\paragraph{Notation.}  
    Let $\cX $ be a countable universe of strings.
    For instance, if $\Sigma$ is a finite alphabet (\eg{}, $\{a, b, \ldots, z\}$), then $\cX=\Sigma^*$ can be the set of all finite-length strings formed by concatenating symbols from $\Sigma$. 
    We define a language $L$ as an \textit{infinite} subset of $\cX$. 
    A countable collection of languages is denoted by $\cL = \inbrace{L_1, L_2, \dots }$. 
    We define a generating algorithm $\generator = (\generator_n)_{n \in \N}$ as a sequence of (possibly randomized) mappings $\generator_n\colon\! {\cX}^n \to 2^{\cX}$ parametrized by the input size $n$. 
    In words, the generator maps a finite training set to a (potentially infinite)\footnote{This is to align with the \emph{set-based} and \emph{element-based} notions of generations that have been considered in the literature.} set of elements.

    \subsection{Language Generation and Identification in the Limit}
    We now formally define language generation in the limit, both in an online and a statistical model.

    \paragraph{Online model.}
        We begin with an extension of the online model that was introduced by \cite{kleinberg2024language}, which handles randomized generators as necessary for DP.
    \begin{definition}[Language Generation in the Limit \citep{kleinberg2024language}]\label[definition]{def:consistentGeneration}
        Let $\cL =  \{L_1, L_2,\dots\}$ be a collection of languages, $\generator ~{ =\inparen{\generator_n}}$ be a generating algorithm, 
        and $K \in \cL$ be some target language.
                A randomized algorithm $\generator$ is said to generate from $K$ in the limit if, for all enumerations of $K$, with probability 1, there is some $n^\star \in \N$ such that for all steps $n \geq n^\star$, the algorithm’s output  satisfies $\generator_n(S_n) \subseteq \inparen{K \setminus S_n}$, where $S_n$ is the set of the first $n$ elements given in the input. The collection $\cL$
                allows for generation in the limit if there is an algorithm  $\generator$ that 
                 generates from $K$ in the limit for any $K \in \cL.$
    \end{definition}
    We remark that \citet{kleinberg2024language} originally studied deterministic generation algorithms; follow-up works studied this natural randomized version, whose analogue has also been studied for identification~\citep{angluin1988identifying,kalavasis2025limits,charikar2025characterization}.
    To gain some intuition about \cref{def:consistentGeneration}, consider the universe $\cX=\Sigma^*$ and the countable collection of \emph{length-threshold} languages
    $\cL=\{L_1,L_2,\ldots\}$ where $L_\ell=\{x\in \Sigma^*: |x|\geq \ell\}$. 
    Suppose the target language is $K=L_{\ell^*}$ for some unknown $\ell^*\in \N$, and the adversary enumerates $K$ as $x_1,x_2,\ldots$. After observing $S_n=\{x_1,\ldots,x_n\}$, we must have $\ell^*\leq \min_{x\in S_n} |x|$. Hence every string of length strictly greater than $\min_{x\in S_n}|x|$ lies in \emph{every} candidate language consistent with $S_n$, and in particular lies in $K$. 
    A valid generator is therefore: for $n\geq 1$, let $m_n=\min_{x\in S_n}|x|$ and output the lexicographically smallest string $y\in \Sigma^{m_n+1}$ with $y\notin S_n$. %
    
    We will also frequently make use of the closure of a language collection, 
    as well as the closure dimension, which characterizes uniform generation,
    defined below.

    \begin{definition}[Closure of Language Collection and Closure Dimension \citep{li2024generation}]\label{def:closureDimension}
        Let $\cL$ be a language collection. The closure of $\cL$, denoted as $\cl(\cL)$, is the intersection of all the languages in $\cL$, \ie{},
        $ \cl(\cL) \coloneqq \bigcap_{L \in \cL} L$.
        The closure dimension of collection $\cL$ is the smallest $d \in \{-1\} \cup \N$ such that for any subcollection $\cL' \subseteq \cL$ of languages, either $\card{\cl(\cL')} = \infty$, or $\card{\cl(\cL')} \leq d$.
    \end{definition}
Throughout this paper, we allow our algorithms access to the languages in the form of a \emph{membership} oracle: for every $i\in \N$ and $x\in \cX $, we can decide whether $x\in L_i$.
Sometimes, we will also allow our algorithms to use the other existing oracles introduced by prior work.

\paragraph{Language identification.}
We now define the preceding notion of language identification.
\begin{definition}[Language Identification in the Limit~\citep{gold1967language}]
                Fix a collection $\cL=\{L_1,L_2,\dots\}$.
                An adversary chooses an unknown target language $K\in \cL$ and enumerates its strings as $x_1,x_2,\dots$ (ensuring that every $x\in K$ appears at some time).
                At each step $n$, the identification algorithm $\mathcal{I}$ observes $x_1,\dots,x_n$ and outputs an index $i_n$ as its current guess for the target.
                We say that $\mathcal{I}$ \emph{identifies $K$ in the limit} if there is a time $n^\star$ after which it never changes its mind and its stabilized guess is correct: for all $n\geq n^\star$ we have $i_n=i_{n^\star}$ and $L_{i_n}=K$.
                The collection $\cL$ is \emph{identifiable in the limit} if there exists an identification algorithm that succeeds for every $K\in\cL$ and every enumeration.
            \end{definition}
            Identification is a strictly stronger requirement than generation and is achievable only for restricted collections. 
            \citet{angluin1980inductive} provided a characterization of which collections are identifiable in the limit (see \cref{def:angliun-criterion}), showing that identifiability imposes stringent structural constraints on the collection.

        \paragraph{Stochastic model of identification.}
        Next, we describe the \emph{stochastic} model of language identification, introduced by \citet{angluin1988identifying} and studied by several follow-up works. 
        Here, the adversary chooses some target $K \in \cL$ and some \emph{distribution}
        $D$ with $\supp(D) = K.$ Then, in every timestep $t \in \N$ a new string
        is drawn i.i.d.\ from $K$ and is revealed to the learner, whose task
        is to figure out the index of the target. 
        Thus, a distribution $D$ is called \emph{valid} if $\supp(D) \in \cL$, \ie, it is entirely supported on a language in $\cL.$
        Naturally, the success criterion for an identification algorithm in this setting is that for every $K \in \cL$ and every $D$ with $\supp(D) = K,$ then the algorithm will make only finitely many mistakes identifying $K$ on an (infinite) i.i.d. stream from $D$, where the probability is both with respect to its internal randomness and the randomness of the stream. 
        The formal definition (\cref{def:stochastic-id}) is deferred to \cref{apx:prelim}.
        Interestingly, \citet{angluin1988identifying} showed that $\cL$ is identifiable in the stochastic setting if and only if it is identifiable in Gold's setting. 

    \enlargethispage{\baselineskip}

\subsection{Differential Privacy and Continual Release}
    Differential privacy~\citep{dwork2006calibrating} is a stability notion for randomized algorithms. 
    Intuitively, it protects users' data by ensuring that the output of the algorithm does not depend too strongly on any single individual's data. 
    \begin{definition}[Pure Differential Privacy]\label[definition]{def:pure-dp}
        Two datasets (or sets of strings) $X,X'\in \cX^n$ (for $n\in \N$) are \emph{neighboring} if they differ in exactly one coordinate.
        Fix an $\eps>0$. 
        A (randomized) algorithm $\generator_n\colon \cX^n \to \DeltaDist{\cX}$ is \emph{$\eps$-DP} if for all neighboring datasets $X$ and $X'$ and all measurable events $\cE \subseteq \DeltaDist{\cX}$, $\Pr\!\big[\generator_n(X)\in \cE\big] \leq e^{\eps}\cdot \Pr\!\big[\generator_n(X')\in \cE\big].$
    \end{definition}

    \noindent As language generation is a continual learning problem, with strings being continually generated, we must ensure that the entire process is private as opposed to a single output.
    This is precisely captured by the \emph{continual release}~\citep{dwork2010continual,chan2011continual} model of differential privacy.
    \begin{definition}[Continual Release]\label[definition]{def:continualRelease}
        Two streams (sequences) of strings $x_{1:n}, x_{1:n}'\in \cX^n$ (for $n\in \N\cup \sinbrace{\infty}$) are \emph{neighboring} if they differ at exactly one timestep.
        Fix an $\eps>0$. 
        A (randomized) algorithm $\generator_n\colon \cX^n \to \DeltaDist{\cX}^n$ that outputs a distribution $\DeltaDist{\cX}_i$ after observing $x_{1:i}$ ($i\in [n]$) is \emph{$\eps$-DP} if for all neighboring streams $x_{1:n}$ and $x_{1:n}'$ and all measurable events $\cE \subseteq \DeltaDist{\cX}^n$, $\Pr\!\big[\generator_n(x_{1:n})\in \cE\big] \leq e^{\eps}\cdot \Pr\!\big[\generator_n(x_{1:n}')\in \cE\big].$
    \end{definition}
    We emphasize that \Cref{def:continualRelease} requires the \emph{entire} output stream to satisfy DP, while \Cref{def:pure-dp} only requires the output at a \emph{single} timestep to satisfy DP.

\section{Proofs of \cref{thm:privateOnlineGeneration,infthm:identification-barriers}}
    In this section, we prove \cref{thm:privateOnlineGeneration,infthm:identification-barriers}; the remaining proofs appear in \cref{apx:proofs}. %

    \subsection{Proof of \Cref{thm:privateOnlineGeneration} (Private Generation for Countable Collections)}\label{apx:proofs:privateGeneration}

    Next, we prove \Cref{thm:privateOnlineGeneration}, which asserts that \Cref{alg:generation:approxIntersection} is $\eps$-DP in the continual release model and generates from any countable collection with probability 1.
    Before proving \Cref{thm:privateOnlineGeneration}, we present a useful lemma that reduces the task of privately generating valid unseen strings from the target language $K$ to computing an infinite subset of $K$.

    \enlargethispage{\baselineskip}
     
    \begin{lemma}\label{lem:elementGenerationFromSetGeneration}
        Let $\generator$ be an $\eps$-DP algorithm in the continual release model that, for any countable collection $\cL$, has the property that, with probability 1, there is some $n^\star\in \N$ after which $\generator$ computes an infinite subset $U_n\subseteq K$ of the target language $K$ for all $n\geq n^\star$.
        Then for any sequence of failure probabilities $\beta_n\in (0, 1)$, there is a data-oblivious postprocessing $M\circ \generator$ that is $\eps$-DP in the continual release model and outputs an unseen element $w_n\in U_n\setminus (x_{1:n}\cup w_{1:n-1})$ from $U_n\subseteq K$ at each $n\geq n^\star$ with probability $1-\beta_n$.
    \end{lemma}

    \begin{proof}[Proof of \Cref{lem:elementGenerationFromSetGeneration}]
        At each time step $n\in \N$, $M$ simply extracts a finite subset $V_n\subseteq U_n$ of size $\card{V_n} = \frac{2n}{\beta_n}$ and samples a uniform random string from $V_n$.
        Since $\card{x_{1:n}\cup w_{1:n-1}}\leq 2n$, this avoids one of the observed strings with probability $1-\beta_n$, as desired.
    \end{proof}

    \noindent We are now ready to prove \Cref{thm:privateOnlineGeneration}.
    \begin{proof}[Proof of \Cref{thm:privateOnlineGeneration}]
        We analyze privacy and utility separately.
        
        \paragraph{Privacy analysis.}
        The algorithm accesses the private stream only when releasing noisy consistency counts $\tilde r_{i,t}$.
        This occurs at sparse steps $t_k = k^{6}$ for $k\in \N$, where it computes the vector of true counts $q^{(k)} \coloneqq (r_{1, t_k}, \dots, r_{k, t_k})$ and adds independent Laplace noise $\Lap(b_k)$ to each coordinate, where $b_k\coloneqq \nfrac{t_k^{1/3}}{\eps_0} = \nfrac{k^3}{\eps_0}$.
        
        Consider two neighboring streams $x_{1:\infty}, x_{1:\infty}'$ differing in exactly one element $x_\tau$. 
        For any specific step $t_k$, the $L_1$-sensitivity of the vector query $q^{(k)}$ is bounded by $\Delta_1(q^{(k)}) = \sum_{i=1}^k \abs{r_{i, t_k}(D) - r_{i, t_k}(D')} \leq k,$
        as removing or changing one element can change the set difference $x_{1:t_k}\setminus L_i$ by at most 1 element for each language $L_i$.
        By simple composition of differential privacy (\Cref{prop:simpleComposition}), the total privacy loss is
        \[
            \eps_{\text{total}} = \sum_{k=1}^\infty \frac{\Delta_1(q^{(k)})}{b_k} = \sum_{k=1}^\infty \frac{k}{k^3/\eps_0} = \eps_0 \sum_{k=1}^\infty \frac{1}{k^2} = \eps_0 \cdot \frac{\pi^2}{6} = \eps.
        \]
        Thus, the algorithm satisfies pure differential privacy.
    
        \paragraph{Utility analysis.}
        We must show that generation in the limit is achieved almost surely. This requires that for large enough $t$, the algorithm selects an infinite set of strings (intersection of languages) contained in the target language $K= L_{i^\star}$.
        $L_{i^\star}$ is consistent with the input stream.
        Intuitively, we show that (1) $L_{i^\star}$ maintains a bounded priority score, and (2) any language $L_j$ with ``high error'' will eventually have a priority score larger than $L_{i^\star}$.
        Define the ``bad'' event at step $t_k=k^{6}$ for language $i \leq k$ as the noise overwhelming the signal:
        \[
            E_{i, k} = \inbrace{ \abs{\tilde r_{i, t_k} - r_{i, t_k}} \geq \frac{t_k}{200 i^2} }.
        \]
        \noindent Using the tail bound for $\Lap(b_k)$, observing $t_k/b_k = k^{6} / (k^3/\eps_0) = \eps_0 k^{3}$, we have: $\Pr[E_{i, k}] = e^{ -\frac{t_k/(200 i^2)}{b_k} } = e^{ - \frac{\eps k^{3}}{200 i^2} }.$
        Since $i \leq k$, we have $k^{3}/i^2 \geq k$. Thus $\Pr[E_{i, k}] \leq \exp(-\Omega(\eps_0 k))$.
        Summing over at most $k^2$ events indexed by $k \geq 1$ and $1 \leq i \leq k$, we see the total failure probability is summable since $\sum\nolimits_{k\geq 1, i\leq k} \Pr\insquare{E_{i, k}}
            \leq \sum\nolimits_{k\geq 1} e^{\frac{-\eps_0 k}{200}} k^2
            < \infty.$
        Now, by the Borel--Cantelli lemma, with probability 1, at most a finite number of bad events occur.
        
        Let $\bar k$ be the largest index such that some $E_{i, k}$ occurs.
        Such a $\bar k$ exists almost surely from our work above.
        We know that $E_{i, k}$ for $k > \bar k, i\leq k$ does not occur.
        Conditioned on the complement of these bad events, the following hold.
         
        \begin{enumerate}[leftmargin=*,itemsep=-2pt]
            \item \textbf{Target Language $L_{i^\star}$:} The true error is $r_{i^\star, t} = 0$. For $t\geq \bar k^{6}$, the observed noisy error is $\tilde r_{i^\star, t} < \nfrac{t}{200 (i^\star)^2}$.
            The condition for incrementing the counter $\tilde N_{i^\star}$ is $\nfrac{\tilde r_{i^\star, t}}{t} > \nfrac{1}{200 (i^\star)^2}$.
            Since $\nfrac{1}{300} < \nfrac{1}{200}$, this condition is never met.
            Thus, $\tilde N_{i^\star}$ stops growing, and its priority $\tilde P_{i^\star}$ is bounded by a constant $P^\star\geq i^\star$.
            
            \item \textbf{High Error Languages:} For $t\geq \bar k^{6}$, we ensure that the following holds
            \begin{align*}
                \frac{r_{i, t}}{t} > \frac{1}{100i^2}
                &\implies \frac{\tilde r_{i, t}}{t} > \frac1{200i^2}
                \qquadand
                \frac{r_{i, t}}{t} \leq \frac{1}{300i^2}
                \implies \frac{\tilde r_{i, t}}{t} \leq \frac1{200i^2}\,. 
            \end{align*}
            Thus, any language violating the error threshold by a small margin will always have its counter incremented, and the counter for any language below the threshold by a small margin eventually stops changing.
        \end{enumerate}

    \begin{algorithm2e}
    \caption{Private Approximate Intersection}\label{alg:generation:approxIntersection}
    
    \KwData{Stream of data elements $x_1, x_2, \dots$ and a language collection $\{L_i\}_{i \geq 1}$}
    \KwResult{Privacy parameter $\eps > 0$}
    
    Initialize consistency counts $\tilde N_i \gets 0$ for all $i$\;

    Set $\eps_0 \gets \sfrac{6\eps}{\pi^2}$\;
    
    \For{$t \gets 1$ \KwTo $\infty$}{
        Receive new string $x_t$ and initialize counter $k\gets \sfloor{t^{\sfrac1{6}}}$\;

        \If{$t = k^{6}$}{
            \For{$i \gets 1$ \KwTo $k$}{
                Compute \textit{true} consistency-count $r_{i,t} \gets \card{x_{1:t} \setminus L_i}$\;
                    
                Compute \textit{noisy} consistency-count $\tilde r_{i,t} \gets \max\!\inbrace{0, r_{i,t} + \Lap(t^{1/3}/\eps_0)}$ %
    
                If noisy count is large, $\tilde r_{i,t}/t > \sfrac{1}{(200 i^2)}$, \mbox{then update consistency count $\tilde N_i \gets \tilde N_i + 1$}\;

                Update priority $\tilde P_i \gets i + \tilde N_i$\; %
            }
        }
    
        Re-order $\inbrace{L_1, \dots, L_k}$ in increasing priority, tie-breaking by index,
        as {$\sinbrace{L_{i_t(1)}, \dots, L_{i_t(k)}}$,
        \ie{}, for each $j\in [k-1]$,
        ensure either $\tilde P_{i_t(j)} < \tilde P_{i_{t}(j+1)}$
        or $\tilde P_{i_t(j)} = P_{i_{t}(j+1)}$ and $i_t(j) < i_{t}(j+1)$}\;
    
        Compute maximal incremental infinite intersection
        $J_t \gets \max\sinbrace{\bar j\in [k]: \card{\cap_{j=1}^{\bar j} L_{i_t(j)}} = \infty}$\; %
        
        Compute ${\bigcap}_{j\leq J_t} L_{i_t(j)} = \sinbrace{z_1, z_2, \dots}$ and output a uniformly random element {$w_n\in \sinbrace{z_1, \dots, z_{200t^3}}$}\;%
    }
    
    \end{algorithm2e}
         
        We argue that for all large enough $t$, languages with priority at most $P^\star$ (which include $L_{i^\star}$) must have summable error.
        Indeed, the set $\cL_{P^\star}\coloneqq \inbrace{L_i: i\leq P^\star}$
        is a finite set containing $L_{i^\star}$.
        Moreover, any $L_j\notin \cL_{P^\star}$ will have priority $\tilde P_j\geq P^\star$ so that it will always come after $L_{i^\star}$.
        By the finiteness of $\cL_{P^\star}$, for sufficiently large $t$, every $L_i\in \cL_{P^\star}$ whose error exceeds $\nfrac1{100i^2}$ infinitely often will have priority exceeding $P^\star$.
        Thus eventually, every language $L_i$ ordered before $L_{i^\star}$ must have summable error at most $\nfrac1{100i^2}$.
        
        Let $\cl(\cL(k))$ denote the intersection of all languages in $\cL(k)\subseteq \cL_{P^\star}$, the collection of languages ordered before $L_{i^\star}$ at step $t_k$, including $L_{i^\star}$ itself. 
        If we show that $\card{\cl(\cL(k))} = \infty$, we are done as the incremental intersection is guaranteed to include $L_{i^\star}$.
        Indeed, as $k\to \infty$,
        \begin{align*}
            \textstyle\card{\cl(\cL(k))}
            &\geq \card{x_{1:t_k}\cap \cl(\cL(k))}
            \geq t_k \inparen{1 - \sum\nolimits_{L_i\in \cL(k)} \frac{r_{i, t_k}}{t_k}} 
            \geq t_k \inparen{1 - \sum\nolimits_{i\geq 1} \frac{1}{100i^2}}
            \geq \frac{t_k}2\,.\textstyle
        \end{align*}
        In particular,
        $\card{\cl(\cL(k))} = \infty$.
    
        Finally, we apply \Cref{lem:elementGenerationFromSetGeneration} to see that sampling a uniform random string among a size $200t^3$ subset of an infinite subset of the target language repeats a seen element with summable probability $\frac1{100t^2}$ and preserves privacy.
        By another application of the Borel--Cantelli lemma, we see that with probability 1, \Cref{alg:generation:approxIntersection} outputs unseen elements after some finite time.
    \end{proof}

    \subsubsection{Non-Uniform Generation Guarantee}
    Next, we explain how the algorithm $\generator$ (\Cref{alg:generation:approxIntersection}) achieves non-uniform generation.
    In particular, for any $\eps > 0$ and $\beta > 0$, any countable collection $\cL$, and any target language $K \in \cL$, there exists $t = t(\eps, \beta, \cL, K)$ such that $\generator$ is $\eps$-DP in the continual release model, and for any enumeration of $K$, generates from $K$ after step $t$ with probability $1 - \beta$.

    \begin{remark}[Non-Uniform Generation]\label{rem:non-uniform-generation}
    Fix any $\eps, \beta > 0$, a collection $\cL$, and a target language $K = L_{i^\star}$. 
        Using the tail bound of Laplace distribution as in utility analysis of the proof above, there exists $t_1 = t_1(\eps, \beta, \cL, K)$ such that with probability at least $1 - \nicefrac{\beta}{2}$, we have $\nicefrac{\tilde{r}_{i^\star, t}}{t} \le \frac{1}{100 {i^\star}^2}$ for all $t \ge t_1$, in which case we have $\tilde{P}_{i^\star} = i^\star + \tilde{N}_{i^\star} \le i^\star + t_1$ and it stays fixed for all $t \ge t_1$. Using the tail bound of Laplace distribution again, there exists $t_2 = t_2(\eps, \beta, \cL, K, t_1)$ such that with probability at least $1 - \nicefrac{\beta}{2}$, we have $\left|\nicefrac{\tilde{r}_{i, t}}{t} - \nicefrac{r_{i, t}}{t}\right| \le \frac{1}{200 {i}^2}$ for all $i \le i^\star + t_1$ and $t \ge t_2$. 
        
        Now, conditional on these events which take place with probability at least $1 - \beta$, there exists $t_3 = t_3(\cL, K, t_1, t_2)$ such that the target language $K$ participates in the maximal incremental infinite intersection at step $t$ for all $t \ge t_3$. To see this, note that for $t_3$ large enough, the priority of the target language $K$ stays fixed and satisfies $\tilde{P}_{i^\star} \le i^\star + t_1$, and all the languages $L_i$ with indices at most $i^\star + t_1$ satisfy $\left|\nicefrac{\tilde{r}_{i, t}}{t} - \nicefrac{r_{i, t}}{t}\right| \le \frac{1}{200 {i}^2}$. Let $B := \max\{|\cl(\cL_S)|: S\subseteq [i^\star + t_1], |\cl(\cL_S)| < \infty\}$ denote the size of the maximum finite intersection of a subcollection of the languages with indices at most $i^\star + t_1$. For $t > 2B$, either all the languages with priorities at most the priority of $K$ have an infinite intersection, in which case we are done and $\generator$ starts generating from $K$ after step $t$, or the languages with priorities at most the priority of $K$ have a finite intersection and $\frac{r_{i,t}}{t} > \frac{1}{100i^2}$ for some ``bad'' language $L_i$ that comes before $K$ in the priority ordering at step $t$. However, in the latter case, the priority of ``bad'' language increments by $1$, and this can only happen for a finite number of steps depending on $i^\star$ and $t_1$, after which we end up in the first case.
    \end{remark}

    \subsection{Proof of \cref{infthm:identification-barriers} (Private Online Identification Lower Bound)}\label{apx:proofs:privateOnlineIdentificationLowerBound}

\begin{proof}[Proof of \Cref{infthm:identification-barriers}]
Fix $\eps>0$ and suppose for contradiction that there exists an $\eps$-DP continual release
identification algorithm $A$ for $\cL$.
Let $L_i,L_j\in\cL$ be distinct such that $|L_i\cap L_j|=\infty$ and $|L_i\setminus L_j|<\infty.$
Set $F  \coloneqq  L_i\setminus L_j,$ $m \coloneqq |F|<\infty,$ $I  \coloneqq  L_i\cap L_j,$ and $V \coloneqq L_j\setminus L_i.$
If $|V|<m$, swap the roles of $(i,j)$: since $m<\infty$ and $L_i\neq L_j$, after possibly swapping
we may assume throughout that
\begin{equation}\label{eq:wlog-V-ge-m}
|V|\geq m \quad (\text{in particular, }V\neq\emptyset).
\end{equation}
This will be useful because enumerations can replace the $m$ elements of $F$ by $m$ \emph{distinct} elements of $V$ while staying duplicate-free. 

\paragraph{Group privacy for continual release.}
By group privacy (\Cref{prop:groupPrivacy}), if $A$ is $\eps$-DP and two streams $x_{1:T},x'_{1:T}$ differ in at most $k$ time steps, then for every event
$\mathcal{E}$ over the first $T$ outputs,
\begin{equation}\label{eq:group-privacy-27-fixed}
\Pr[A(x)_{1:T}\in\mathcal{E}]
~\le~
e^{k\eps}\,\Pr[A(x')_{1:T}\in\mathcal{E}].
\end{equation}
Order $\cX$ canonically. Further, enumerate $F=\{f_1,\dots,f_m\}$ and $I=\{a_1,a_2,\dots\}$ in canonical order
and define a duplicate-free enumeration of $L_i$: $E  \coloneqq  (f_1,\dots,f_m,\ a_1,a_2,a_3,\dots).$
Since $A$ identifies $L_i$ on every (duplicate-free) enumeration, given $E$, with probability $1$, $A$ outputs the correct $i$ all but finitely many times. In particular, for $N_j^S(T)  \coloneqq  \big|\{t\leq T:\ A \text{ outputs index } j \text{ at time }t\text{ on input stream }S\}\big|,$
we have
\begin{equation}\label{eq:Nj-on-E-goes-to-0}
\Pr\big[N_j^E(T)\geq T/2\big]\xrightarrow[T\to\infty]{}0.
\end{equation}
Consider a canonical enumeration of  $V=L_j\setminus L_i$, \ie, $V=\{u_1,u_2,\dots\}$.
By \eqref{eq:wlog-V-ge-m}, $u_1,\dots,u_m$ exist and are distinct.
Define $E^{(0)}$ by replacing the first $m$ elements of $E$ with $u_1,\dots,u_m$:
\[
E^{(0)}  \coloneqq  (u_1,\dots,u_m,\ a_1,a_2,a_3,\dots).
\]
Then $E^{(0)}$ is duplicate-free and every element of $E^{(0)}$ lies in $L_j$.
Moreover, $E$ and $E^{(0)}$ differ in exactly $m$ positions, so applying \eqref{eq:group-privacy-27-fixed}
to the event $\{N_j(\infty)=\infty\}$, we get that
$A$ outputs $j$ only finitely many times almost surely on input $E^{(0)}$ as well. Hence,
\begin{equation}\label{eq:Nj-on-E0-goes-to-0}
\Pr\big[N_j^{E^{(0)}}(T)\geq T/2\big]\xrightarrow[T\to\infty]{}0.
\end{equation}
Now define  $\delta_k  \coloneqq  \frac{e^{-2k\eps}}{k^2}$.
Hence, it holds that 
\[
    \sum_{k=1}^\infty \delta_k e^{k\eps}
=
\sum_{k=1}^\infty \frac{e^{-k\eps}}{k^2}
<\infty.
\]
By \eqref{eq:Nj-on-E0-goes-to-0}, we can choose an increasing sequence of times
$T_1<T_2<\cdots$ such that for all $k\geq 1$,
\begin{equation}\label{eq:Tk-choice-fixed}
\Pr\big[N_j^{E^{(0)}}(T_k)\geq T_k/2\big]\leq \delta_k.
\end{equation}
We now perform an infinite sequence of \emph{single-coordinate} edits at the times $T_k$ that turns
$E^{(0)}$ into an enumeration of $L_j$, while ensuring that up to time $T_k$ we changed at most $k$
positions (so we can apply group privacy with parameter $k$).
Let $U^{(0)}  \coloneqq  L_j \setminus \{E^{(0)}_t : t\geq 1\}.$
Concretely, $U^{(0)}$ contains exactly the ``still-missing'' elements of $V$, namely
$U^{(0)}=\{u_{m+1},u_{m+2},\dots\}$ (possibly empty if $|V|=m$).
We define inductively streams $E^{(k)}$ and pools $U^{(k)}$ as follows.
Assume $E^{(k-1)}$ has been defined, is duplicate-free and contains only elements in $L_j$.
If $U^{(0)}=\emptyset$, then $E^{(0)}$ already enumerates $L_j$ (it contains all of $V$ and all of $I$),
and we may set $E'' \coloneqq E^{(0)}$ and skip the subsequent steps.
Otherwise, for each $k\geq 1$:
\begin{itemize}[itemsep=-2pt]
    \item Let $v_k$ be the \emph{smallest} element of $U^{(k-1)}$ in the canonical order.
    \item Let $y_k  \coloneqq  E^{(k-1)}_{T_k}$ be the element currently occupying position $T_k$.
    \item Define $E^{(k)}$ by a single replacement at time $T_k$: $E_t^{(k)}$ is $v_k$ if $t=T_k$ and, otherwise, it is $E_t^{(k-1)}$.
    \item Update the pool by reverting the insertion and deletion: $U^{(k)}  \coloneqq  \big(U^{(k-1)}\setminus\{v_k\}\big)\ \cup\ \{y_k\}.$
\end{itemize}
Next, we prove that this maintains duplicate freeness and correctness of the pool.
We claim by induction on $k$:
\begin{enumerate}[itemsep=-2pt]
    \item $E^{(k)}$ is duplicate-free and $E^{(k)}_t\in L_j$ for all $t$.
    \item $U^{(k)} = L_j \setminus \{E^{(k)}_t:t\geq 1\}$ (\ie{}, $U^{(k)}$ is exactly the set of elements of
    $L_j$ still missing from the current stream).
\end{enumerate}
This is immediate: by the inductive hypothesis, $U^{(k-1)}$ is disjoint from the range of $E^{(k-1)}$, so
$v_k\notin\{E^{(k-1)}_t\}$ and inserting $v_k$ introduces no duplicate; simultaneously we remove $y_k$ from
the stream and add it back to the pool, preserving both disjointness and the identity $U^{(k)}=L_j\setminus \mathrm{range}(E^{(k)})$.

Now define the limiting stream $E''$ as $v_k$ if $t=T_k$ for some $k$ and, otherwise, define it as $E_t^{(0)}$.
Since the $T_k$'s are strictly increasing, each coordinate is modified at most once, so $E''$ is well-defined.

\paragraph{$E''$ enumerates $L_j$.}
From the invariant $U^{(k)} = L_j\setminus \mathrm{range}(E^{(k)})$ and the fact that once a value is placed
at coordinate $T_k$ it is never changed again, we get the following dichotomy for any $x\in L_j$:
either $x$ is never placed out and it stays in the final stream, or it is placed out once (when it equals some
$y_k$) and then it enters the pool. Because at each phase we insert the \emph{smallest} element of the pool,
and because the canonical order is induced by an enumeration of $\cX$ (so each element has finitely many
predecessors), every fixed $x\in L_j$ can be bypassed only finitely many times before it becomes the smallest
pool element and is inserted at some later phase. Once inserted, it is never placed out again. Therefore every
$x\in L_j$ appears in $E''$ at some finite index, and $E''$ is a duplicate-free enumeration of $L_j$.

For each $k$, consider the event $\evF_k  \coloneqq  \big\{N_j^{E''}(T_k)\geq T_k/2\big\}.$
By construction, the prefixes $E''_{1:T_k}$ and $E^{(0)}_{1:T_k}$ differ in \emph{exactly} the $k$ positions
$T_1,\dots,T_k$, hence in at most $k$ positions. Applying group privacy \eqref{eq:group-privacy-27-fixed}
at horizon $T_k$ and then \eqref{eq:Tk-choice-fixed} yields
\[
\Pr[\evF_k\ \text{under input }E'']
\ \le\
e^{k\eps}\cdot \Pr\big[N_j^{E^{(0)}}(T_k)\geq T_k/2\big]
\ \le\
e^{k\eps}\delta_k.
\]
Since $\sum_{k\geq 1} e^{k\eps}\delta_k<\infty$, the first Borel--Cantelli lemma implies that
with probability $1$ only finitely many events $\evF_k$ occur when $A$ is run on input $E''$.

However, if $A$ identified $L_j$ on the valid enumeration $E''$, then with probability $1$ there would exist
a time $\tau$ such that $A$ outputs $j$ at every round $t\geq \tau$. Then for all $k$ with $T_k\geq 2\tau$, $N_j^{E''}(T_k)\ \ge\ T_k-\tau\ \ge\ T_k/2,$
so $\evF_k$ would occur for all sufficiently large $k$, and hence infinitely often, which is a contradiction.

Therefore, $A$ cannot identify $L_j$ on the enumeration $E''$, contradicting the assumption that $A$
identifies $\cL$ in the limit. This completes the proof.
\end{proof}

\section{Conclusion}
In this work we initiate the study of privacy in language generation and identification in the limit. 
Surprisingly, online generation remains achievable under strong privacy constraints, whereas online identification is severely restricted.
Unlike the online setting, in the stochastic model of \citet{angluin1988identifying}, private identification becomes achievable for all collections which are identifiable without privacy.
This reveals a strong separation between private online and stochastic identification, which is absent in non-private settings. 
Our work suggests several future directions:
including investigating more lenient variants of differential privacy~\citep{bun2016concentrated,mironov2017renyi}, exploring the interplay between privacy and \emph{breadth} \citep{kalavasis2025limits,kalavasis2024characterizations,charikar2024facets,kleinberg2025density,kleinberg2025partial,peale2024}, and studying if private algorithms can be designed for uncountable collections.

\subsection*{Acknowledgments}
We thank anonymous reviewers for comments that helped improve the presentation of this work.
Felix Zhou acknowledges the support of the Natural Sciences and Engineering Research Council of Canada (NSERC). Xifan Yu is supported in part by ONR Award N00014-24-1-2611.

\clearpage
\printbibliography

\appendix

\crefalias{section}{appendix}
\crefalias{subsection}{subappendix}
\crefalias{subsubsection}{subsubappendix}

\newpage

\section{Additional Preliminaries}\label{apx:prelim}
In this section, we present some additional preliminaries.

\subsection{Characterization of Language Identification in the Limit}
    \citet{angluin1980inductive} provided a condition that characterizes the subset of countable collections which are identifiable in the limit. 
    Informally, a collection satisfies Angluin's condition if for any language $L \in \cL$, there exists a finite subset $T_L$ (called a tell-tale set) that serves as a finite ``fingerprint'' allowing one to distinguish $L$ from any other language that contains $T_L$. 
    \begin{restatable}[Angluin's Condition \citep{angluin1980inductive}]{definition}{angluinsCondition}\label[definition]{def:angliun-criterion}
        \label{def:telltales}
                Fix a language collection $\cL = \{L_1, L_2, \dots\}$.
                The collection $\cL$ is said to satisfy Angluin's condition if for any index $i$, there is a tell-tale, \ie{}, a finite set of strings $T_i$ such that $T_i$ is a subset of $L_i$, \ie{}, $T_i\subseteq L_i$, and the following holds:
                \vspace{-1mm}
                \begin{center}
                    \centering 
                    For all $j\geq 1$, if $L_j\supseteq T_i$, then $L_j$ is not a proper subset of $L_i$.
                \end{center}
    \end{restatable}
    \noindent Roughly, this condition ensures that after observing enough examples from the target language, one can rule out all incorrect languages. 
    The main result of \citet{angluin1980inductive} is as follows:
    \begin{theorem}
    [Characterization of Identification in the Limit \citep{angluin1980inductive}]
    \label[theorem]{thm:angluin-id-limit}
    The following holds for any countable collection of languages $\cL.$
    \begin{enumerate}[leftmargin=15pt]
        \item $\cL$ is identifiable in the limit if it satisfies Angluin's condition and one has access to the tell-tale oracle.
        \item If there is an algorithm that identifies $\cL$ in the limit, then Angluin's condition is true and the tell-tale oracle can be implemented.
    \end{enumerate}
    \end{theorem}
    The above tight characterization shows that language identification is information-theoretically impossible even for simple collections of languages, such as the collection of all regular languages. 
    {Crucially, access to the tell-tale oracle is necessary for identification in the limit (its existence alone is not sufficient); see Theorem 2 in \cite{angluin1980inductive}.} 
   
    \subsection{Stochastic Identification in the Limit}
    In this section, we formally define language identification in the limit in a stochastic setting. 

    \begin{definition}[Stochastic Identification in the Limit]\label{def:stochastic-id}
        Fix a collection $\cL=\{L_1,L_2,\dots\}$.
        An adversary chooses an unknown target language $K\in \cL$ and a distribution $D$ supported on $K$.
        At each step $n$, the identification algorithm $\mathcal{I}$ observes $x_1,\dots,x_n\sim_{\iid} D$ and outputs an index $i_n$ as its current guess for the target.
        We say that $\mathcal{I}$ \emph{identifies $K$ in the limit} if there is a time $n^\star$ after which it never changes its mind and its stabilized guess is correct: for all $n\geq n^\star$ we have $i_n=i_{n^\star}$ and $L_{i_n}=K$.
        The collection $\cL$ is \emph{identifiable in the limit} if there exists an identification algorithm that succeeds for every $K\in\cL$ and every distribution $D$ supported on $K$.
    \end{definition}

    \begin{remark}[Achieving Identification with Randomness]
                Gold's model of language identification in the limit requires the learner to eventually stabilize on a single correct index $i^\star$.
                At first glance, this is in tension with differential privacy, since any non-trivial DP learner must randomize and therefore outputs an incorrect index with positive probability.
                This, however, can be resolved: it suffices to ensure that the probability of outputting an incorrect index at round $n$ is summable over $n$.
                The Borel--Cantelli lemma then implies that, with probability $1$, only finitely many incorrect outputs occur, so the learner stabilizes to the correct index outside a null event.
            \end{remark}

\subsection{Borel--Cantelli Lemma} 
Next, we present a well-known result due to Borel and Cantelli which is useful for ensuring our private algorithms only make a finite number of ``mistakes'' with probability 1.

\begin{lemma}[First Borel--Cantelli Lemma]\label{lem:first-Borel--Cantelli}
    Let $\inbrace{\cE_n}_{n \in \N}$ be a sequence of events. If $\sum_{n \in \N} \Pr[\cE_n] < \infty,$ then the probability that infinitely many of them occur is 0, that is $\Pr\insquare{\limsup_{n \rightarrow \infty} \cE_n} = 0.$
\end{lemma}
The previous result has a \textit{partial} converse, which we omit here as we do not need it. %

    \subsection{Privacy Tools}
    Some useful properties of DP include composition, post-processing, and group privacy.
    \begin{proposition}[Simple Composition; \citet{dwork2014algorithmic}]\label{prop:simpleComposition}
        Let $M_1: \cX^*\to \cY, M_2: \cX^*\times \cY\to \cZ$ be $\eps_1$-DP and $\eps_2$-DP, respectively.
        Then the composition $M_2(\cdot, M_1(\cdot)): \cX^*\to \cZ$ is $(\eps_1+\eps_2)$-DP.
    \end{proposition}
    \vspace{-7mm}
    \begin{proposition}[Post-Processing; \citet{dwork2014algorithmic}]\label{prop:postProcessing}
        Let $M: \cX^*\to \cY$ be $\eps$-DP and $f: \cY\to \cZ$ be any data-independent function.
        Then $f(M(\cdot)): \cX^*\to \cZ$ is $\eps$-DP.
    \end{proposition}
    \vspace{-7mm}
    \begin{proposition}[Group Privacy; \citet{dwork2014algorithmic}]\label{prop:groupPrivacy}
        Let $M: \cX^*\to \cY$ be $\eps$-DP.
        For all datasets $X, X'$ that differ by at most $k\geq 1$ elements, and all measurable events $\cE \subseteq \DeltaDist{\cY}$, 
        \[
            \Pr\!\big[M(X)\in \cE\big] \leq e^{k\eps}\cdot \Pr\!\big[M(X')\in \cE\big]\,.
        \]
    \end{proposition}
    One of the most ubiquitous tools for pure DP is the exponential mechanism.
    \begin{theorem}[Exponential Mechanism; \citet{talwar2007exponential}]\label{thm:exp-mechanism}
        Let $R$ be a collection of elements and $u: \cX^*\times R\to \R$ a score function with sensitivity $\Delta_u$ across neighboring datasets.
        Then the following exponential mechanism preserves $\eps$-DP:
        select an element $r\in R$ with probability %
        $\propto~\exp\inparen{\frac{\eps\cdot u(X, r)}{2\Delta_u}}.$
    \end{theorem}
    In fact, the standard Laplace mechanism for numerical queries can be viewed as a special case of the exponential mechanism.
    \begin{proposition}[Laplace Mechanism; \citet{dwork2014algorithmic}]\label{prop:laplaceMechanism}
        Let $f: \cX^*\to \R$ be a numerical query with sensitivity $\Delta_f$ across neighboring datasets.
        Then the Laplace mechanism, which outputs $f(X) + \Lap(\Delta_f/\eps)$, preserves $\eps$-DP.
    \end{proposition}

\section{Additional Related Work}\label{apx:related-work}

    Below we overview some additional works related to generation in the limit \citep{kleinberg2024language}
            \begin{itemize}[leftmargin=12pt]
                \item \textbf{Robustness to Noise:}
                    While the model of \citet{kleinberg2024language} assumes that the adversary introduces no errors or omissions in the input stream, recent work has relaxed this requirement.
                    \citet{raman2025noisy} allow the adversary to introduce a finite number of errors in the input stream and show that generation in the limit remains possible for all countable collections. 
                    \citet{bai2025noise} allow the adversary to omit elements of the target language from the stream and, as a corollary, show that all countable collections remain generatable even with an infinite number of omissions.
                    \citet{mehrotra2025contamination} extend both of these directions, considering a model where the adversary can introduce both forms of contamination (insert ``noisy'' elements and omit elements from the target language) and show that all countable collections remain generatable even with an infinite amount of contamination, provided the frequency of noise is ``controlled.''
                    Our private generation algorithm builds on a method of \citet{mehrotra2025contamination}, and interestingly inherits the same tolerance to contamination; in particular, our algorithm is both private and robust to noisy inputs and omissions.

                \item \textbf{Language Generation with Breadth:}
                The algorithm of \citet{kleinberg2024language} eventually outputs only in-language strings (and hence eventually stops outputting elements outside of $K$), but this can come at the cost of \emph{breadth}---the ability to generate diverse strings from the target language.
                A number of works formalize breadth in different ways and show that many natural breadth requirements make generation significantly harder, in some cases approaching the difficulty of identification \citep{kalavasis2025limits,charikar2024facets,kalavasis2024characterizations,peale2024,kleinberg2025density}.
                Our results also connect to this direction: our identification algorithms can be converted into generation algorithms achieving these breadth notions, using our private subroutine for sampling uniformly from a language (see \cref{lem:elementGenerationFromSetGeneration}). In a related direction, \citet{peng2026language} showed that if one has access to the computational trace of a machine that accepts the underlying language, then \emph{identification} in the limit (which is perhaps the strongest notion of breadth), is achievable for all collections that are accepted by Turing Machines.
            \end{itemize}

\section{Deferred Proofs}
    \label{apx:proofs}
    
    \subsection{Proof of \Cref{infthm:ubFinite} (Upper Bound on Sample Complexity)}\label{apx:proofs:privateUniformGenerationUB}
        \label{sec:proofof:thm:ubFinite}

    Here, we prove \cref{infthm:ubFinite}, which says that for any collection $\cL$ of $k$ languages with closure dimension $d$, there exists an $\eps$-DP algorithm in the continual release model, such that for any $\beta > 0$, it generates from step $n^*$ onward from $\cL$ for $n^* = d + \tilde{O}\left(\left(\nfrac{k}{\eps}\right)\log\left(\nfrac{1}{\beta}\right)\right)$ with probability at least $1 - \beta$.
    First, we state the formal version of \cref{infthm:ubFinite} and then prove it.
    \begin{theorem}[Sample Complexity Sufficient for Uniform Private Generation]
        \label{thm:ubFinite}
        Let $\cL = \{L_1, \dots, L_k\}$ be a collection of languages with closure dimension $d$. 
        \begin{itemize}[]
            \item \textbf{(Continual Release DP)} There is an $\eps$-DP generation algorithm $\generator$ in the continual release model such that for any $m \in \N$, 
            target language $K\in \cL$,
            and input enumeration,
            the time step $n^\star$ after which $\generator$ generates from $K$ satisfies $\Pr[n^\star \leq m]
            \geq 1 - \exp\left(- \Omega\left(\left(\nfrac{\eps}{k}\right) \cdot \nfrac{(m-d)}{\log^2 (m-d)}\right)\right)$.
            \item \textbf{(DP)} There is an $\eps$-DP generation algorithm $\generator$ that, for any target language $K\in \cL$, given any finite set of $n$ input elements, generates an unseen element from $K$ with probability at least $1 - 5\exp\left(- \nfrac{\eps(n-d)}{(2k)}\right)$.
        \end{itemize}
    \end{theorem}
    \begin{proof}[Proof of \cref{thm:ubFinite}]
    We will first give a generation algorithm that is $\eps$-DP on a finite set of $n$ input elements, and then use it to obtain an $\eps$-DP generation in the continual release model.
    
    Assume that $\cL = \{L_1, \dots, L_k\}$ is a collection of languages with closure dimension $d$. For a subset $S \subseteq [k]$ of indices, let $\cL_S = \{L_i: i \in S\}$ denote the subcollection of languages indexed by $S$. We will use $\cl(\cL_S) = \bigcap_{i \in S} L_i$ to denote the closure of a subcollection of languages.
    
    \paragraph{Upper bound for finite sample.}
    Consider the following exponential mechanism, which assigns score to a subcollection $\cL_S$ given seen examples $x_{1:n}$. Concretely, for any $S \subseteq [k]$, we set
    \begin{align*}
        u(S, x_{1:n}) \coloneqq \card{\cl(\cL_S) \cap x_{1:n}} + f(n)\cdot |S|\,,
    \end{align*}
    where $f(n)$ is some quantity that we will set later. We will sample $\cL_S$ with probability proportional to $\exp\left(\frac{\eps \cdot u(S, x_{1:n})}{2 \Delta u}\right)$, where $\Delta u$ is the global sensitivity of $u$, which in this case is $1$. This exponential mechanism is $\eps$-pure DP.
    
    For a set $S$, we call it good if the index $i^\star$ of the target language $K$ is contained in $S$, \ie{}, $i^\star \in S$, and $\card{\cl(\cL_S)} = \infty$. We call a set $S$ bad if $\card{\cl(\cL_{S \cup \{i^\star\}})} < \infty$. We call a set $S$ conservative if $i^\star \not\in S$ and $\card{\cl(\cL_{S \cup \{i^\star\}})} = \infty$. We note that any $S$ falls into exactly one of the three categories above. Moreover, if $S$ is conservative, then $S \cup \{i^\star\}$ must be good. We will denote 
    \begin{align*}
        s(\text{good}) &= \sum_{\text{good } S} \exp\left(\frac{\eps \cdot u(S, x_{1:n})}{2}\right)\,,\\
        s(\text{bad}) &= \sum_{\text{bad } S} \exp\left(\frac{\eps \cdot u(S, x_{1:n})}{2}\right)\,,\\
        s(\text{conservative}) &= \sum_{\text{conservative } S} \exp\left(\frac{\eps \cdot u(S, x_{1:n})}{2}\right)\,.
    \end{align*}
    First, let us consider the bad sets. If $S$ is bad, then $\card{\cl(\cL_{S \cup \{i^\star\}})} < \infty$ implies that $\card{\cl(\cL_{S \cup \{i^\star\}})} \leq d$ by consideration of the closure dimension. Thus, 
    \begin{align*}
        u(S, x_{1:n}) &= \card{\cl(\cL_S) \cap x_{1:n}} + f(n)\cdot |S|\\
        &=  \card{\cl(\cL_S) \cap K \cap x_{1:n}} + f(n)\cdot |S|\\
        &\leq \card{\cl(\cL_{S \cup \{i^\star\}})} + f(n)\cdot k\\
        &\leq d + k\cdot f(n)\,.
    \end{align*}
    Next, let us consider the good sets. We know that
    \begin{align*}
        \max_{\text{good } S} u(S, x_{1:n}) \geq u(\{i^\star\}, x_{1:n})
        &\geq n + f(n)\,.
    \end{align*}
    Since there are at most $2^k$ bad sets, we have
    \begin{align*}
        s(\text{bad}) &\leq 2^k \cdot \exp\left(\frac{\eps(d + k\cdot f(n))}{2}\right)
        = \exp\left(\frac{\eps(d + k\cdot f(n))}{2} + k \log 2\right)\,.
    \end{align*}
    Since $S \cup \{i^\star\}$ must be good if $S$ is conservative, we have
    \begin{align*}
        s(\text{conservative}) &= \sum_{\text{conservative } S} \exp\left(\frac{\eps(\card{\cl(\cL_S) \cap x_{1:n}} + f(n) \cdot \card{S})}{2}\right)\\
        &= \exp\left(- \frac{\eps \cdot f(n)}{2}\right)\cdot \sum_{\text{conservative } S} \exp\left(\frac{\eps(\card{\cl(\cL_S) \cap K\cap x_{1:n}} + f(n) \cdot \card{S \cup \{i^\star\}})} {2}\right)\\
        &= \exp\left(- \frac{\eps \cdot f(n)}{2}\right)\cdot \sum_{\text{conservative } S} \exp\left(\frac{\eps(\card{\cl(\cL_{S\cup \{i^\star\}}) \cap x_{1:n}} + f(n) \cdot \card{S \cup \{i^\star\}})} {2}\right)\\
        &\leq \exp\left(- \frac{\eps \cdot f(n)}{2}\right)\cdot s(\text{good})\,.
    \end{align*}
    Finally, we have
    \begin{align*}
        s(\text{good}) &\geq \max_{\text{good } S} \exp\left(\frac{\eps \cdot u(S, x_{1:n})}{2}\right)
        \geq \exp\left(\frac{\eps \cdot (n + f(n))}{2}\right)\,.
    \end{align*}
    Therefore, we know that the probability of sampling a good set $S$ using this exponential mechanism is at least
    \begin{align*}
        P(\text{good}) &= \frac{s(\text{good})}{s(\text{bad}) + s(\text{conservative}) + s(\text{good})}\\
        &\geq \frac{1}{\exp\left(\frac{\eps(d + k\cdot f(n))}{2} + k \log 2 - \frac{\eps \cdot (n + f(n))}{2}\right) + \exp\left(- \frac{\eps \cdot f(n)}{2}\right) + 1}\\
        &\geq 1 - \exp\left(\frac{\eps(d + k\cdot f(n))}{2} + k \log 2 - \frac{\eps \cdot (n + f(n))}{2}\right) - \exp\left(- \frac{\eps \cdot f(n)}{2}\right)\,.
    \end{align*}
    Now we set $f(n) \coloneqq \frac{1}{k}\left(n - d - \frac{2k\log 2}{\eps}\right)$, with which we get $P(\text{good}) \geq 1 - 4 \exp\left(- \frac{\eps(n-d)}{2k} \right)$.
    
    In particular, outputting $\cl(\cL_S)$ where $S \subseteq [k]$ is sampled according to the above exponential mechanism is $\eps$-DP at time $n$, which satisfies that w.p. at least $1 - 4 \exp\left(- \frac{\eps(n-d)}{2k} \right)$,
    \[|\cl(\cL_S)| = \infty \text{ and } \cl(\cL_S) \subseteq K \,.\]
    \noindent By \cref{lem:elementGenerationFromSetGeneration}, we may choose $\beta_n = \exp\left(-\frac{\eps(n-d)}{2k}\right)$ to obtain an element-based generator that is $\eps$-DP at time $n$ and outputs an element in $\cl(\cL_S)$ distinct from the $n$ input elements with probability at least $1 - \exp\left(-\frac{\eps(n-d)}{2k}\right)$. Combined with the guarantee for $\cl(\cL_S)$, given $n$ distinct input elements $x_1, \dots, x_n$ from $K$, this $\eps$-DP generator outputs an element $o_n \in K \setminus \{x_1, ..., x_n\}$ with probability at least $1 - 5\exp\left(-\frac{\eps(n-d)}{2k}\right)$.
    
    \paragraph{Upper bound for continual release model.}
    Finally, we convert the above differentially private generator in the finite sample setting into a generator that is differentially private in the continual release model. To do so, for $t = 1, 2, \dots$, we define
    \begin{align*}
        \eps_t = \frac{6}{\pi^2} \cdot \frac{\eps}{t^2}
        \qquadand
        n_t = 2^t+d\,.
    \end{align*}
    At each step $n = n_t$ for some $t \in \N$, we apply the exponential mechanism with privacy parameter $\eps_t$ to sample a set $\cl(\cL_{S_t})$ such that with probability at least $1 - 4 \exp\left(-\frac{\eps_t(n_t-d)}{2k}\right)$, we have
    \begin{align*}
        |\cl(\cL_{S_t})| = \infty \qquadtext{and} \cl(\cL_{S_t}) \subseteq K\,.
    \end{align*}
    By \cref{lem:elementGenerationFromSetGeneration}, we may apply postprocessing to $\cl(\cL_{S_t})$ to output elements in $\cl(\cL(S_t))$ distinct from the input stream for all steps between $n_t$ and $n_{t+1} - 1$. This ensures that with probability at least $1 - \exp\left(-\frac{\eps_t(n_t-d)}{2k}\right)$.
    
    By simple composition \cref{prop:simpleComposition}, the total privacy budget of this algorithm in the continual release model is then at most
    \begin{align*}
        \sum_{t = 1}^\infty \eps_t = \frac{6}{\pi^2} \sum_{t=1}^\infty \frac{\eps}{t^2} \leq \eps\,,
    \end{align*}
    and this confirms that this algorithm is $\eps$-DP in the continual release model. By union bound, we also know that the probability that the algorithm outputs from $K \setminus \{x_1, \dots, x_n\}$ for all $n \geq n_t$ onward is at least
    \begin{align*}
        &\quad 1 - \sum_{t' \geq t} \left(4\exp\left(-\frac{\eps_{t'}(n_{t'}-d)}{2k}\right) + \exp\left(-\frac{\eps_{t'}(n_{t'}-d)}{2k}\right)\right)\\
        &\geq 1 - 5\sum_{t' \geq t} \exp\left(-\frac{6}{\pi^2} \cdot\frac{\eps2^{t'}}{2t'^2 k}\right)\\
        &= 1 - 5\sum_{t' \geq t} \exp\left(-\frac{3}{\pi^2} \cdot\frac{\eps}{ k} \cdot \frac{2^{t'}}{t'^2}\right)\\
        &\geq 1 - \exp\left(-\Omega\left(\frac{\eps ((n_t - d)/\log^2(n_t - d))}{k}\right)\right)\,.
    \end{align*}
    Since for any $m \geq d + 2$, there exists $t \in \N$ such that $n_t - d\leq m - d \leq 2(n_t - d)$, we conclude that for any $m \in \N$, this algorithm generates from $K$ from step $n^\star$ onward for some $n^\star \leq m$ with probability at least
    \begin{align*}
        1 - \exp\left(-\Omega\left(\frac{\eps ((m - d)/\log^2(m - d))}{k}\right)\right)\,.
    \end{align*}
    This finishes the proof.
    
    \end{proof}

    \subsection{Proof of \Cref{infthm:lb} (Lower Bound on Sample Complexity)}\label{apx:proofs:privateUniformGenerationLB}

    Here we prove \cref{infthm:lb}, which shows the necessity of the dependency on $d+\,\nfrac{k}{\eps}$ for the sample complexity proved in \cref{thm:ubFinite}.
    \begin{remark}[Closure Dimension] \label{rem:lbFinite-closure-dimension}
    The language collection constructed in the proof of \cref{infthm:lb} has closure dimension $0$.
    Indeed, the intersection of any sub-collection of $\cL$ with size $\ell$ is infinite if $\ell\leq \floor{\nfrac{k}2}$, or 0 otherwise. Thus, $\cL$ is generatable with a single sample.
    We also note that we may easily incorporate the closure dimension $d$ in our lower bound construction. 
    The easiest way is to append a common set of $d$ elements to all the languages in the constructed collection in \cref{infthm:lb}. In the data sets $x_{1:n}$ and $y_{1:n}$ we construct for the proof, we will always set the first $d$ elements in both data sets to be the $d$ common elements of all the languages. In this way, we may show that we need $n \geq \frac{1}{\eps}(k \log 2 - O(\log k)) + d$, in order for an $\eps$-DP algorithm to generate from $K$ at time $n$ with probability at least $\nfrac{2}{3}$.
\end{remark}
    Due to the above remark, without loss of generality, we can focus on the special case of \cref{infthm:lb} with $d=0$.
    We first state the formal version of \cref{infthm:lb} (in this special case) and then prove it.
    \begin{restatable}[Tightness of Sample-Complexity for Uniform Private Generation]{theorem}{DPsampleLB}
            \label{thm:DP-uniform-generation-barrier-finite-collection}
            \label{thm:lbFinite}
            There exists a collection of $k$ languages $\cL = \{L_1, \dots, L_k\}$ such that 
                for any $\eps$-DP generation algorithm $\generator$ in the continual release model (\cref{def:continualRelease}), 
                if the random time $n^\star$ such that $\generator$ generates from step $n^\star$ onward satisfies
                $\Pr_{n^\star}[n^\star\leq m]\geq \nfrac23$, then $m\geq \frac{1}{\eps}(k \log 2 - O(\log k))$.
                Further, without requiring privacy, there is a generation algorithm $\generator$ that is guaranteed to generate from $\cL$ after step $n^\star=1$.
        \end{restatable}
        The collection witnessing \Cref{thm:lbFinite} is defined in the following way. Let $N = \binom{k}{\lfloor \sfrac{k}{2}\rfloor}$. Let us enumerate the $\lfloor \nfrac{k}{2}\rfloor$-subsets of $[k]$ as $\{S_1, S_2, \dots, S_N\}$. Define the $\cL$ as the collection consisting of
        $            L_i = \{j + Nt \,\vert\, S_j\ni i, t \in \N\} \subseteq \N, \text{ for } i \in [k].
$
        
            We remark that our lower bound in \cref{thm:lbFinite} also applies to the finite sample guarantee. For the same collection of languages $\cL = \{L_1, \dots, L_k\}$, if an $\eps$-DP generator $\cA$ generates correctly at time $n$ with probability at least $\nfrac{2}{3}$ for any $K$ and any enumeration, then $n \geq \frac{1}{\eps}(k \log 2 - O(\log k))$.

\begin{algorithm2e}[htbp]
\caption{Data-Independent Epoch Exponential Mechanism}\label{alg:finite_intersections_em}
\KwData{Stream of distinct elements $x_1,x_2,\dots$; collection $\cL=\{L_i\}_{i\geq 1}$; overlaps $M(k) \coloneqq \max_{1 \leq a < b \leq k} |L_a \cap L_b|$ (with $M(1) \coloneqq 0$); privacy $\eps>0$}
\KwResult{Continual-release hypotheses $\widehat L^t$ for all $t\geq 1$}

\BlankLine
Set privacy split $\eps_s \gets \frac{6\eps}{\pi^2 s^2}$ for $s\geq 1$ \tcp*[r]{$\sum_s \eps_s=\eps$}

\BlankLine
Initialize epoch $s\gets 1$\;
Output $\widehat L^1 \gets L_1$ \tcp*[r]{Initialize first output}

\For{$t\gets 1$ \KwTo $\infty$}{
    Receive $x_t$\;
    Set next release time $t_s \gets 2^s$\;
    \If{$t=t_s$}{
        Set active search space $W_s \gets \max \left( \{1\} \cup \left\{ d \leq s : M(d) \leq \frac{t_s}{2} \right\} \right)$ \tcp*[r]{Data-independent cap}
        \ForEach{$i \in \{1, \dots, W_s\}$}{
            $\mathrm{Err}_{t_s}(i) \gets \sum_{r\leq t_s}\mathds{1}[x_r\notin L_i]$ \tcp*[r]{Error count of language $i$}
            $u_s(i) \gets - \mathrm{Err}_{t_s}(i)$ \tcp*[r]{Utility function $u_s(i) \leq 0$}
        }
        Set sensitivity $\Delta \gets 1$\;
        Set temperature $\lambda_s \gets \eps_s/(2\Delta)$\;
        Sample $I_s \in \{1, \dots, W_s\}$ according to the exponential mechanism:
        \Indp
        $\Pr[I_s=i \mid X_{1:t_s}] \propto \exp(\lambda_s u_s(i))$\;
        \Indm
        \For{$\tau \gets t_s$ \KwTo $t_{s+1}-1$}{
            Output $\widehat L^\tau \gets L_{I_s}$ \tcp*[r]{Repeat output between releases}
        }
        Increment epoch $s \gets s+1$\;
    }
}
\end{algorithm2e}

    \begin{proof}[Proof of \Cref{thm:lbFinite}]
        Consider the collection of languages $\{L_1, \dots, L_k\}$ defined in the following way. Let $N = \binom{k}{\lfloor \sfrac{k}{2}\rfloor}$. Let us enumerate the $\lfloor \nfrac{k}{2}\rfloor$-subsets of $[k]$ as $\{S_1, S_2, \dots, S_N\}$. Define the languages as
        \begin{align*}
            L_i = \{j + Nt \,\vert\, S_j\ni i, t \in \N\} \subseteq \N, \text{ for } i \in [k]\,.
        \end{align*}
        We will also denote $\cL_{S_i} = \{L_i: i\in S_i\}$.
        
        Note that by design, we have
        \begin{align*}
            \cl(\cL_{S_i}) = \bigcap_{j \in S_i} L_j
            = \bigcap_{j \in S_i} \{\ell + Nt \, \vert \, S_\ell \ni j, t \in \N\} 
            = \{\ell + Nt \, \vert \, S_\ell\supseteq S_i, t \in \N\}
            = \{i + Nt \, \vert \, t \in \N\}\,,
        \end{align*}
        where in the last equality we use the fact that $\{S_1, \dots, S_N\}$ is a Sperner family, \ie{}, $S_i \not\subseteq S_j$ for any $i \ne j$. 
    
        \paragraph{Lower bound for finite sample.}
        We will first show a stronger lower bound, that any $\eps$-DP algorithm on a finite set of elements $x_{1:n}$ needs $n \geq \frac{1}{\eps}(k \log 2 - O(\log k))$ in order to generate from the target language with probability at least $\nfrac{2}{3}$ at step $n$. Suppose $\cA: \N^\star \to \N$ is an element-based generator that is $\eps$-DP, and suppose that $\cA$ generates from the target language with probability at least $\nfrac{2}{3}$ at step $n$. Next, we proceed to show a lower bound for $n$.
        
        Consider the following post-processing of $\cA$. Define $f: \N \to \{1, \dots, N\}$ as $f(i) \equiv i \text{ mod } N$. Note that $B = f \circ \cA: \N^\star \to \{1, \dots, N\}$ is again $\eps$-DP by post processing \cref{prop:postProcessing}. Let $x_{1:n} = \{x_1, \dots, x_n\}$ be an arbitrary data set. Let $j \in \{1, \dots, N\}$ be the minimizer of $\Pr(B(x_{1:n}) = j)$. Note that we have $\Pr(B(x_{1:n}) = j) \leq \nfrac{1}{N}$.
    
        On the other hand, let us consider an alternative data set $y_{1:n} = \{y_1, \dots, y_n\}$ with distinct elements such that $y_i \equiv j \text{ mod } N$ for all $i \in [n]$. In other words, we have $y_{1:n} \subseteq \{j + Nt \, \vert \, t \in \N\}$. Since $B$ is $\eps$-DP, by \cref{prop:groupPrivacy}, we have
        \begin{align}
            \Pr(B(y_{1:n}) = j) \leq \exp(n \eps) \cdot \Pr(B(x_{1:n}) = j) \leq \frac{\exp(n \eps)}{N}\,. \label{ineq:dp-lower-bound}
        \end{align}
        Note that since $y_{1:n } \subseteq \{j + Nt \, \vert \, t \in \N\} = \cl(\cL_{S_i})$ is the prefix of some valid enumeration of all languages in $\cL_{S_i}$ simultaneously, for $\cA$ to generate from the target language with probability at least $\nfrac{2}{3}$ on the data set $y_{1:n}$, its output must be in the intersection $\cl(\cL_{S_i})$ with probability at least $\nfrac{2}{3}$. Therefore, with probability at least $\nfrac{2}{3}$, we have
        \begin{align*}
            A(y_{1:n}) &\in \cl(\cL_{S_i}) = \{j + Nt \, \vert \, t \in \N\}
            \qquadand
            B(y_{1:n}) = f(A(y_{1:n})) = j\,.
        \end{align*}
        Combining with \eqref{ineq:dp-lower-bound}, we get
        \begin{align*}
            \frac{2}{3}\leq \Pr(B(y_{1:n}) = j) \leq \exp(n \eps) \cdot \Pr(B(x_{1:n}) = j) \leq \frac{\exp(n \eps)}{N}\,,
        \end{align*}
        and thus
        \begin{align*}
            n &\geq \frac{1}{\eps}\log\left(\frac{2}{3}N\right)
            = \frac{1}{\eps}\log\left(\frac{2}{3}\binom{k}{\lfloor \sfrac{k}{2}\rfloor}\right)
            = \frac{1}{\eps}\left(k \log 2 - O(\log k)\right)\,.
        \end{align*}
        This concludes the proof that for the constructed collection $\cL$, if an $\eps$-DP algorithm $\cA$ on a finite set $x_{1:n}$ of $n$ input elements generates from $K$ with probability at least $\nfrac{2}{3}$, then $n \geq \frac{1}{\eps}\left(k \log 2 - O(\log k)\right)$.
    
        \paragraph{Lower bound for continual release model.}
        We can now easily lift our lower bound for the finite sample guarantee to the continual release model, as the latter is a stronger requirement. Suppose $\generator$ is an $\eps$-DP generation algorithm in the continual release model. If the random time $n^\star$ such that $\generator$ generates from step $n^\star$ onward satisfies $\Pr[n^\star \leq m] \geq \nfrac{2}{3}$, then in particular, $\generator$ needs to generate at step $m$ with probability at least $\nfrac{2}{3}$. Moreover, since $\generator$ is $\eps$-DP in the continual release model, it is also $\eps$-DP on a finite set $x_{1:m}$ of $m$ input elements. By our lower bound for the finite sample guarantee, we have $m \geq \frac{1}{\eps}(k \log 2 - O(\log k))$ as desired.
    \end{proof}
    Next, using \cref{thm:lbFinite}, we may construct a countable collection $\cL$ with closure dimension $0$, such that for any finite $n$, no private algorithm can generate from $\cL$ at time $n$ with probability at least $\nfrac{2}{3}$.
    \begin{corollary}\label{cor:lbInfinite}
        There exists a countable language collection $\cL$ with closure dimension $0$, such that for any $n \in \N$, no $\eps$-DP algorithm can generate from $K$ at time $n$ with probability at least $\nfrac{2}{3}$ for arbitrary $K$ and enumeration of $K$.
    \end{corollary}
    Thus, the difference in sample complexity between uniform private generation and uniform non-private generation can not only be arbitrarily large, as shown by \cref{thm:lbFinite}, it can also be infinite.
    \begin{proof}[Proof of \Cref{cor:lbInfinite}]
        Let $\cL_k$ be the finite collection of $k$ languages constructed in \cref{thm:DP-uniform-generation-barrier-finite-collection}. Consider the countable collection of languages $\cL$ defined as
        \begin{align*}
            \cL \coloneqq \bigsqcup_{k \in \N} \left\{L \times \{k\}: L \in \cL_k\right\}.
        \end{align*}
        Note that any language in $\cL$ is an infinite set in $\N^2$. Moreover, since each $\cL_k$ has closure dimension $0$, it is clear that $\cL$ also has closure dimension $0$.
    
        Assume for contradiction that there exists $n \in \N$ and an $\eps$-DP algorithm that generates from $K$ at time $n$ with probability at least $\frac{2}{3}$ for arbitrary $K$ and its enumeration. In particular, for any subcollection $\{L \times \{k\}: L \in \cL_k\} \subseteq \cL$, this algorithm must generate from $K$ at time $n$ with probability at least $\frac{2}{3}$ for arbitrary $K \in \{L \times \{k\}: L \in \cL_k\}$ and its enumeration. Note that this subcollection is isomorphic to $\cL_k$, and thus by \cref{thm:DP-uniform-generation-barrier-finite-collection}, we have $n\geq \frac{1}{\eps}\left(k \log 2 - O(\log k)\right)$. Since $n\geq \frac{1}{\eps}\left(k \log 2 - O(\log k)\right)$ must hold for arbitrary $k \in \N$, we arrive at a contradiction and conclude that there is no such $n \in \N$.
    \end{proof}

    \subsection{Proof of \cref{thm:finite_intersection_upper} (Private Online Identification Upper Bound)}\label{apx:proofs:privateOnlineIdentificationUpperBound}
    \begin{theorem}[Upper Bound]
    \label{thm:finite_intersection_upper}
    Let $\cL = \{L_1, L_2, \dots\}$ be a countably infinite collection of infinite languages and $\eps > 0$. \cref{alg:finite_intersections_em} satisfies $\eps$-DP in the continual release model and,  if $\cL$ has finite pairwise intersections,  identifies $\cL$ in the limit in the online setting.
    \end{theorem}
    
\begin{proof}[Proof of \Cref{thm:finite_intersection_upper}]
We prove the privacy and correctness guarantees of our algorithm separately.

\paragraph{Privacy.} 
Differential privacy requires the mechanism to be stable against changes in worst-case streams. We analyze the sensitivity of the utility function $u_s(i)$ at epoch $s$. Consider two neighboring infinite streams $X$ and $X'$ that differ in exactly one coordinate (a single replacement). The prefixes $X_{1:t_s}$ and $X'_{1:t_s}$ will differ in at most one element. Therefore, the error count $\mathrm{Err}_{t_s}(i) = \sum_{r=1}^{t_s} \mathds{1}[x_r \notin L_i]$ changes by at most $1$. Thus, the global $\ell_1$-sensitivity is strictly bounded by $\Delta = 1$.

Crucially, the active search space $W_s$ depends only on the public function $M(\cdot)$ and the deterministic epoch length $t_s$. It is entirely independent of the private data stream $X$. Thus, restricting the domain of the exponential mechanism to $W_s$ does not consume any privacy budget.

By the standard guarantee of the exponential mechanism (\Cref{thm:exp-mechanism}), the release of $I_s$ at epoch $s$ satisfies pure $\eps_s$-DP. Because the epochs operate on nested prefixes of the same stream, we apply basic sequential composition over the infinite horizon. The total privacy cost is $\sum_{s=1}^\infty \eps_s = \sum_{s=1}^\infty \frac{6\eps}{\pi^2 s^2} = \eps$. Since the intra-epoch outputs $\widehat L^\tau$ are formed by deterministically repeating the most recently sampled $I_s$, post-processing ensures the entire output transcript satisfies pure $\eps$-CR-DP.

\paragraph{Correctness.} 
Utility is evaluated on valid stream enumerations, which by definition in the online setting contain no duplicate elements. Fix the true target language $K = L_{i^\star}$. We will show that the probability of the exponential mechanism selecting any incorrect index $i \neq i^\star$ is summable over $s$.

Because $M(i^\star)$ is a finite constant and $t_s = 2^s \to \infty$, there exists some epoch $s_1$ such that for all $s \geq s_1$, $M(i^\star) \leq t_s / 2$ and $s \geq i^\star$. Therefore, for all $s \geq s_1$, the target index satisfies the condition for the active set, meaning $i^\star \leq W_s$. 

Because the adversary's stream is a valid enumeration of $L_{i^\star}$, every element $x_r \in L_{i^\star}$. Thus, for all $s$, the true utility of the target is perfectly zero: $u_s(i^\star) = 0$.

Consider any epoch $s \geq s_1$ and any other active candidate $i \leq W_s$ where $i \neq i^\star$. By the definition of the active set $W_s$, we are guaranteed that $M(W_s) \leq t_s / 2$. 

The maximum number of elements the candidate $L_i$ can share with the target $L_{i^\star}$ is $|L_{i^\star} \cap L_i| \leq M(\max(i^\star, i))$. Since both $i^\star \leq W_s$ and $i \leq W_s$, we have $\max(i^\star, i) \leq W_s$. Because $M(\cdot)$ is non-decreasing, $|L_{i^\star} \cap L_i| \leq M(W_s) \leq t_s / 2$.

Since the stream consists of $t_s$ \emph{distinct} elements from $L_{i^\star}$, at most $t_s / 2$ of these elements can also belong to $L_i$. Consequently, $L_i$ must be inconsistent with at least $t_s - t_s/2 = t_s/2$ elements in the stream prefix. 
Therefore, its utility is strictly bounded: $u_s(i) \leq -t_s/2 = -2^{s-1}$.

For any $s \geq s_1$, the probability of selecting an incorrect hypothesis is bounded by comparing the weights of all incorrect hypotheses against the weight of the true target $i^\star$. Let $Z_s = \sum_{j=1}^{W_s} \exp(\lambda_s u_s(j))$ be the normalization factor. Since $u_s(i^\star) = 0$, we have $Z_s \geq \exp(0) = 1$.
\begin{align*}
\Pr[I_s \neq i^\star \mid X_{1:t_s}] &= \sum_{\substack{i=1: i \neq i^\star}}^{W_s} \frac{e^{\lambda_s u_s(i)}}{Z_s} 
\leq \sum_{\substack{i=1: i \neq i^\star}}^{W_s} e^{-\lambda_s 2^{s-1}} 
\leq W_s e^{-\lambda_s 2^{s-1}} 
\leq s e^{-\lambda_s 2^{s-1}}\,.
\end{align*}
In the last step, we used the algorithmic constraint that $W_s \leq s$. Substituting $\lambda_s = \eps_s/(2\Delta) = \frac{3\eps}{\pi^2 s^2}$, the probability of making a mistake at epoch $s$ is bounded by:
\[ \Pr[I_s \neq i^\star \mid X_{1:t_s}] \leq s \exp\left( - \frac{3\eps}{\pi^2 s^2} 2^{s-1} \right). \]
Because the exponential decay inside the argument vastly overpowers the polynomial term $s$, this probability decays super-polynomially fast and is unconditionally summable over $s$. Thus, $\sum_{s=1}^\infty \Pr[I_s \neq i^\star \mid X_{1:t_s}] < \infty$. 
By the Borel--Cantelli lemma, the event $\{I_s \neq i^\star\}$ occurs only finitely many times almost surely. Hence, there exists an epoch $s_0$ such that $I_s = i^\star$ for all $s \geq s_0$. The algorithm makes finitely many mistakes and successfully identifies $L_{i^\star}$ in the limit.
\end{proof}

    \subsection{Proof of \cref{thm:stochastic-identification} (Private Stochastic Identification)}\label{apx:proofs:privateStochasticIdentification}
\begin{algorithm2e}[htbp]
\caption{Private Stochastic Identification}\label{alg:dp-telltale-em}
\KwData{Stream $x_1,x_2,\dots$; collection $\cL=\{L_i\}_{i\geq 1}$; tell-tales $\{T_i\}_{i\geq 1}$; prior $\pi=(\pi_i)_{i\geq 1}$ with $\pi_i>0$ and $\sum_i \pi_i=1$; privacy parameter $\eps>0$}
\KwResult{Continual-release hypotheses $\widehat L^t$ for all $t\geq 1$}

\BlankLine
Set privacy split $\eps_s \gets \frac{6\eps}{\pi^2 s^2}$ for $s\geq 1$
\tcp*[r]{$\sum_s \eps_s=\eps$}

\BlankLine
Initialize counts $c\gets 0$ on $\cX$ \tcp*[r]{$c(w)$ maintains $c_t(w)$ online}
Initialize epoch $s\gets 1$\;

\For{$t\gets 1$ \KwTo $\infty$}{
    Receive $x_t$\;
    Update count $c(x_t)\gets c(x_t)+1$\;
    Set next release time $t_s \gets 2^s$\;
    \If{$t=t_s$}{
        Set thresholds $k_s \gets s^3$\;
        \ForEach{$i\geq 1$}{
            $\mathrm{Err}_{t_s}(i) \gets \sum_{r\leq t_s}\mathds{1}[x_r\notin L_i]$ \tcp*[r]{Error count of language $i$}
            $\mathrm{Def}_{t_s,s}(i) \gets \sum_{w\in T_i}\max\{0,\,k_s-c(w)\}$ \tcp*[r]{Deficit count of language $i$}
            $u_s(i) \gets -\mathrm{Err}_{t_s}(i)-\mathrm{Def}_{t_s,s}(i)$ \tcp*[r]{Utility function $u_s(i)\leq 0$}
            $\pi_s(i) \gets \pi(i) \cdot s^{-2i}$ \tcp*[r]{Update base measure}
        }
        Set sensitivity $\Delta \gets 3$\;
        Set temperature $\lambda_s \gets \eps_s/(2\Delta)$\;
        Sample $I_s$ according to the exponential mechanism with base measure $\pi_s$\;
        \Indp
        $\Pr[I_s=i \mid X_{1:t_s}] \propto \pi_s(i) \exp(\lambda_s u_s(i))$\;
        \Indm
        \For{$\tau \gets t_s$ \KwTo $t_{s+1}-1$}{
            Output $\widehat L_{\tau} \gets L_{I_s}$ \tcp*[r]{Repeat output between releases}
        }
        Increment epoch $s \gets s+1$\;
    }
}
\end{algorithm2e}

\noindent Recall that if $\cL$ does not identify Angluin's condition, then it is \emph{not} identifiable even in the absence of privacy constraints. Hence, the lower bound follows immediately; we focus
on obtaining the upper bound.

First, we establish the privacy guarantees of the algorithm by bounding the sensitivity of the utility function and composing the privacy loss across all epochs.

\begin{lemma}[Privacy]
\label{lem:crdp}
For any $\varepsilon>0$, 
\cref{alg:dp-telltale-em} appropriately parametrized is $\varepsilon$-differentially private in the continual release model.
\end{lemma}
\vspace{-7mm}
\begin{proof}
We analyze the privacy guarantee in three steps: bounding the global sensitivity of the utility function, establishing the privacy of each individual epoch, and composing the privacy loss across the infinite stream.

\paragraph{Step 1: Global sensitivity of the utility function.}
Consider two neighboring stream prefixes $X_{1:t_s}$ and $X'_{1:t_s}$ that differ in exactly one element (representing a single replacement). We analyze the maximum effect this replacement can have on the utility $u_s(i) = -\mathrm{Err}_{t_s}(i) - \mathrm{Def}_{t_s,s}(i)$. 

Replacing one element changes the error count $\mathrm{Err}_{t_s}(i) = \sum_{r\leq t_s} \mathds{1}[x_r \notin L_i]$ by at most $1$. For the deficit term $\mathrm{Def}_{t_s,s}(i) = \sum_{w\in T_i}\max\{0,\,k_s-c(w)\}$, replacing an element decreases the frequency count of one symbol by $1$ and increases the count of another by $1$. Because the function $c \mapsto \max\{0, k_s - c\}$ is $1$-Lipschitz, the deficit sum changes by at most $1+1=2$. By the triangle inequality, the global sensitivity of $u_s(i)$ is bounded by $\Delta = 1 + 2 = 3$.

\paragraph{Step 2: Epoch-level privacy.}
At the end of each epoch $s$, the algorithm selects an index $I_s$ via the exponential mechanism, sampling proportional to $\pi_s(i) \exp(\lambda_s u_s(i))$. The time-dependent base measure $\pi_s(i) \coloneqq \pi_i s^{-2i}$ depends only on the public prior $\pi$ and the deterministic epoch index $s$; it is entirely independent of the private data stream. Therefore, modifying the base measure dynamically does not consume any privacy budget. By setting the temperature parameter to $\lambda_s = \eps_s / (2\Delta)$, the standard guarantee of the exponential mechanism (\Cref{thm:exp-mechanism}) ensures $\eps_s$-DP.

\paragraph{Step 3: Continual release via composition.}
Fix an arbitrary time horizon $T \in \N$. The continuous transcript of outputs up to time $T$, denoted $(\widehat L^1, \dots, \widehat L^T)$, is a deterministic post-processing of the finite sequence of epoch indices $(I_1, \dots, I_m)$, where \mbox{$m = \max\{s : t_s \leq T\}$}. 

Because the algorithm processes nested prefixes of the same underlying data stream, we apply the basic composition theorem for differential privacy (\Cref{prop:simpleComposition}). The joint release of the indices $(I_1, \dots, I_m)$ satisfies $(\sum_{s=1}^m \eps_s)$-DP. The algorithm's privacy budget is explicitly split such that $\eps_s = \frac{6\eps}{\pi^2 s^2}$. Thus, the total privacy loss over all epochs is strictly bounded by the convergent infinite series $\sum_{s=1}^\infty \eps_s = \eps$. 

Since the sequence of indices $(I_1, \dots, I_m)$ is $\eps$-DP, and the step-by-step hypotheses $\widehat L_\tau$ for \mbox{$\tau \in [t_s, t_{s+1}-1]$} are formed by deterministically repeating these indices ($\widehat L_\tau = L_{I_s}$), the post-processing property of differential privacy (\Cref{prop:postProcessing}) ensures that the entire output transcript satisfies $\eps$-DP.
\end{proof}
Having established the privacy guarantees, we now
shift to discussing the correctness of our approach.
Fix the target index $i^\star$ and distribution $D$ with $\supp(D)=L_{i^\star}$.

\begin{lemma}[Correctness]\label{lem:stochastic-online-algo-correctness}
    For any collection of 
    languages $\cL$
    that satisfies Angluin's condition,
    \cref{alg:dp-telltale-em} identifies $\cL$ in the limit
    from stochastic examples.
\end{lemma}

\vspace{-7mm}
\begin{proof}
Fix the target index $i^\star$ and the target distribution $D$ with $\supp(D)=L_{i^\star}$. We will show that the algorithm makes finitely many mistakes almost surely.

\paragraph{Step 1: The target language eventually has zero deficit.}
Let the tell-tale of $L_{i^\star}$ be $T_{i^\star}=\{w_1,\dots,w_m\}$ and let $p_j \coloneqq D(w_j)>0$. For each $j$, the stream count $c_{t_s}(w_j)$ follows a binomial distribution $\mathrm{Bin}(2^s, p_j)$. Since the deficit threshold is $k_s = s^3$, for all sufficiently large $s$ we have $k_s = s^3 \leq (p_j/2) 2^s$. By a Chernoff bound,
\[
\Pr[c_{t_s}(w_j) < k_s] \leq \Pr[c_{t_s}(w_j) < (p_j/2)\,2^s] \leq \exp(-p_j 2^s/8).
\]
Let $A_s \coloneqq \{\mathrm{Def}_{t_s,s}(i^\star)=0\}$ be the event that the target language has zero deficit at epoch $s$. Taking a union bound over the finite tell-tale $T_{i^\star}$, we have $\Pr[A_s^c] \leq \sum_{j=1}^m \exp(-p_j 2^s/8)$. Because this decays exponentially in $2^s$, the sum of probabilities is finite: $\sum_{s=1}^\infty \Pr[A_s^c] < \infty$.

\paragraph{Step 2: Pointwise bounds on the exponential mechanism.}
Conditioned on the stream $X_{1:t_s}$, the exponential mechanism samples $I_s$ with probability proportional to $\pi_i s^{-2i} \exp(\lambda_s u_s(i))$. Let $Z_s$ be the normalization factor. On the event $A_s$, the target language has perfect utility $u_s(i^\star) = 0$ (since $\supp(D)=L_{i^\star}$ implies $\mathrm{Err}_{t_s}(i^\star) = 0$ always). Therefore, $Z_s \geq \pi_{i^\star} s^{-2i^\star} \exp(0) = \pi_{i^\star} s^{-2i^\star}$. 

For any incorrect language $i \neq i^\star$, we can bound the conditional probability of selecting it on the event $A_s$ as follows:
\begin{align}
\Pr[I_s=i \mid X_{1:t_s}] \mathds{1}_{A_s} 
&\leq \frac{\pi_i s^{-2i} \exp(\lambda_s u_s(i))}{\pi_{i^\star} s^{-2i^\star}} \cdot \mathds{1}_{A_s} 
= \frac{\pi_i}{\pi_{i^\star}} s^{2(i^\star-i)} \exp(\lambda_s u_s(i)) \cdot \mathds{1}_{A_s}\,. \label{eq:pointwise_bound}
\end{align}
To show that the algorithm eventually stops making mistakes, we will show that the sum over all epochs and all incorrect languages of the expected probability of making a mistake is finite. We split the sum over $i \neq i^\star$ into the infinite tail ($i > i^\star$) and the finite prefix ($i < i^\star$).

\paragraph{Step 3: Bounding the infinite tail ($i > i^\star$).}
Since utilities are always non-positive, $\exp(\lambda_s u_s(i)) \leq 1$. For any $i > i^\star$, we have $i^\star - i \leq -1$, which implies $s^{2(i^\star-i)} \leq s^{-2}$. Summing \eqref{eq:pointwise_bound} over all $i > i^\star$ yields:
\[
\sum_{i > i^\star} \Pr[I_s=i \mid X_{1:t_s}] \mathds{1}_{A_s} \leq \sum_{i > i^\star} \frac{\pi_i}{\pi_{i^\star}} s^{-2} \leq \frac{s^{-2}}{\pi_{i^\star}} \sum_{i=1}^\infty \pi_i = \frac{s^{-2}}{\pi_{i^\star}}\,.
\]
Taking the expectation over the stream $X$, the sum over all epochs $s$ of this tail bound is $\sum_{s=1}^\infty \frac{s^{-2}}{\pi_{i^\star}} < \infty$.

\paragraph{Step 4: Bounding the finite prefix ($i < i^\star$).}
Since there are only finitely many such indices, we can analyze each fixed $i < i^\star$ individually. Taking the expectation of \eqref{eq:pointwise_bound} over the stream gives:
\begin{equation}
\E\Big[ \Pr[I_s=i \mid X_{1:t_s}] \mathds{1}_{A_s} \Big] \leq \frac{\pi_i}{\pi_{i^\star}} s^{2(i^\star - i)} ~~ \E\!\Big[\exp(\lambda_s u_s(i)) \mathds{1}_{A_s}\Big]\,. \label{eq:prefix_bound}
\end{equation}
We bound the inner expectation by considering two subcases for $L_i$:
\begin{itemize}[leftmargin=*]
    \item \emph{Case 4a: $L_i \not\supseteq L_{i^\star}$.} Then $p_i \coloneqq \Pr_{x\sim D}[x\notin L_i]>0$, and the error is distributed as $\mathrm{Err}_{t_s}(i)\sim \mathrm{Bin}(2^s,p_i)$. Since $\mathrm{Def}_{t_s,s}(i)\geq 0$, we have $u_s(i) \leq -\mathrm{Err}_{t_s}(i)$. Bounding via the moment generating function of the Binomial distribution: 
    \[
    \E[\exp(\lambda_s u_s(i))] \leq \E[\exp(-\lambda_s \mathrm{Err}_{t_s}(i))] = (1-p_i+p_ie^{-\lambda_s})^{2^s} \leq \exp(-p_i(1-e^{-\lambda_s})2^s).
    \]
    Recall $\lambda_s=\eps_s / (2\Delta) = \Theta(1/s^2)$. For all sufficiently large $s$, $1-e^{-\lambda_s}\geq \lambda_s/2$, meaning the expectation is bounded by $\exp(-\Omega(2^s/s^2))$.
    
    \item \emph{Case 4b: $L_i \supsetneq L_{i^\star}$.} By Angluin's condition (\cref{def:telltales}), it must be that $T_i \not\subseteq L_{i^\star}$ (otherwise $T_i \subseteq L_{i^\star} \subsetneq L_i$ implies $L_{i^\star} = L_i$, a contradiction). Thus, there exists some $w_i\in T_i\setminus L_{i^\star}$. Because $\supp(D) = L_{i^\star}$, $w_i$ is never drawn in the stream, meaning $c_{t_s}(w_i)=0$ deterministically. This forces the deficit to be at least $\mathrm{Def}_{t_s,s}(i)\geq k_s=s^3$, yielding $u_s(i) \leq -s^3$. Thus,
    \[
    \E[\exp(\lambda_s u_s(i))] \leq \exp(-\lambda_s s^3) = \exp(-\Omega(s))\,.
    \]
\end{itemize}
In both subcases, the expected exponential utility penalty $\E[\exp(\lambda_s u_s(i))]$ decays at least exponentially fast in $s$. Because the leading time-penalty inversion $s^{2(i^\star - i)}$ in \eqref{eq:prefix_bound} grows only polynomially, the exponential decay strictly dominates. Therefore, for each fixed $i < i^\star$, the expected probability is $O(e^{-\Omega(s)})$, which is summable over $s$. Since there are only finitely many $i < i^\star$, their finite sum satisfies $\sum_{s=1}^\infty \sum_{i < i^\star} \E\Big[ \Pr[I_s=i \mid X_{1:t_s}] \mathds{1}_{A_s} \Big] < \infty$.

\paragraph{Step 5: Conclusion.}
By the law of total probability, we can combine the bounds from the complement event, the infinite tail, and the finite prefix to obtain the unconditional probability of an error:
\begin{align*}
\sum_{s=1}^\infty \Pr[I_s \neq i^\star] 
&= \sum_{s=1}^\infty \Pr[I_s \neq i^\star, A_s^c] + \sum_{s=1}^\infty \Pr[I_s \neq i^\star, A_s] \\
&\leq \sum_{s=1}^\infty \Pr[A_s^c] + \sum_{s=1}^\infty \E\Bigg[ \sum_{i \neq i^\star} \Pr[I_s = i \mid X_{1:t_s}] \mathds{1}_{A_s} \Bigg] \\
&= \sum_{s=1}^\infty \Pr[A_s^c] + \sum_{s=1}^\infty \E\Bigg[ \sum_{i > i^\star} \Pr[I_s = i \mid X_{1:t_s}] \mathds{1}_{A_s} \Bigg] + \sum_{i < i^\star} \sum_{s=1}^\infty \E\Big[ \Pr[I_s = i \mid X_{1:t_s}] \mathds{1}_{A_s} \Big] \\
&< \infty.
\end{align*}
By the Borel--Cantelli lemma, the event $\{I_s \neq i^\star\}$ occurs only finitely many times almost surely. Hence, with probability $1$, there exists some epoch $s_0$ such that for all $s \geq s_0$, $I_s = i^\star$. This means $\widehat L^t = K$ for all $t \geq t_{s_0}$, concluding the proof of identification in the limit.
\end{proof}
We now have all ingredients
to prove \cref{thm:stochastic-identification}.

\begin{proof}[Proof of \cref{thm:stochastic-identification}]
    First, notice that if $\cL$
    does not identify Angluin's condition, it is not identifiable
    in the limit in the online setting \citep{angluin1980inductive}. Moreover, identification in the limit in the online setting
    is equivalent to identification
    in the limit in the stochastic setting \cite{angluin1988identifying}. Thus, it suffices to show the
    other direction.

    \cref{lem:crdp} shows that for any $\eps > 0,$
    \cref{alg:dp-telltale-em} satisfies $\eps$-DP in the continual release model. Then,
    \cref{lem:stochastic-online-algo-correctness} shows that
    \cref{alg:dp-telltale-em} identifies in the limit from
    stochastic examples.
\end{proof}

\end{document}